\documentclass[12pt]{article}
\usepackage[toc,page]{appendix}
\usepackage[left=2cm,right=2cm,top=2cm,bottom=3cm]{geometry}
\usepackage{listings}
\usepackage[utf8]{inputenc}
\usepackage{mathtools}
\usepackage{amsmath,amssymb,amsfonts,amsthm,bbm,mathrsfs,verbatim,bm} 
\usepackage[misc]{ifsym}
\usepackage{graphics}                 
\usepackage{graphicx}
\usepackage{caption}
\usepackage{subcaption}
\usepackage{float}
\usepackage{array}
\usepackage{xurl}
\usepackage{hyperref}
\hypersetup{colorlinks,allcolors=black}
\usepackage[thinc]{esdiff}
\usepackage{makeidx}
\usepackage{algorithm}
\usepackage{graphicx}
\usepackage[shortlabels]{enumitem}
\usepackage{csquotes}
\usepackage{amssymb}
\usepackage[nottoc]{tocbibind}
\usepackage{tcolorbox}
\usepackage{mdframed}
\usepackage{pdfpages}
\graphicspath{ {images/} } 

 \usepackage{blindtext}
\newtheorem{theorem}{Theorem}[section]
\newtheorem{proposition}[theorem]{Proposition}
\usepackage{comment}
\usepackage{lipsum}
\usepackage{xr}
\usepackage{blindtext}
\usepackage[T1]{fontenc}
\usepackage[utf8]{inputenc}
\usepackage{blindtext}
\usepackage[T1]{fontenc}
\usepackage[utf8]{inputenc}
\renewcommand{\thesection}{\arabic{section}}
\usepackage{tikz}
\usetikzlibrary{positioning}
\newdimen\nodeDist
\usepackage{authblk}
\usepackage[utf8]{inputenc}
\usepackage[english]{babel}
 \usepackage{amsmath}
\usepackage{caption}
\usepackage{dsfont}

 \usepackage{algorithm}
\usepackage{algorithmic}
\usepackage{float}
\usepackage{lipsum}
\usepackage{macros}
\usepackage{xcolor}
\usepackage{lastpage}
\usepackage{natbib}
\usepackage{comment}
\title{SBAMDT: Bayesian Additive Decision Trees with Adaptive Soft Semi-multivariate Split Rules}

\author[1]{Stamatina Lamprinakou}
\author[1]{Huiyan Sang}
\author[2]{Bledar A. Konomi}
\author[3]{Ligang Lu}

\affil[1]{Department of Statistics, Texas A\&M University, College Station, USA \\
          \texttt{matina@tamu.edu, huiyan@stat.tamu.edu}}
\affil[2]{Department of Mathematical Sciences, University of Cincinnati, Cincinnati, OH, USA \\
          \texttt{konomibr@ucmail.uc.edu}}
\affil[3]{Shell International Exploration and Production Inc, Houston, TX, USA \\
          \texttt{ligang.lu@shell.com}}

\begin{document}
\maketitle

\begin{abstract}  Bayesian Additive Regression Trees \citep[BART,][]{chipman2010bart} have gained significant popularity due to their remarkable predictive performance and ability to quantify uncertainty. However, standard decision tree models rely on recursive data splits at each decision node, using deterministic decision rules based on a single univariate feature. This approach limits their ability to effectively capture complex decision boundaries, particularly in scenarios involving multiple features, such as spatial domains, or when transitions are either sharp or smoothly varying. In this paper, we introduce a novel probabilistic additive decision tree model that employs a soft split rule. This method enables highly flexible splits that leverage both univariate and multivariate features, while also respecting the geometric properties of the feature domain. Notably, the probabilistic split rule adapts dynamically across decision nodes, allowing the model to account for varying levels of smoothness in the regression function. We demonstrate the utility of the proposed model through comparisons with existing tree-based models on synthetic datasets and a New York City education dataset.\end{abstract}



\section{Introduction}
\label{sec:intro} 

We consider a non-parametric regression problem with response $Y\in R$ and a number of features with known multivariate structures, $s \in \mathcal{M}$ (structured) and unknown structures, $\bfx \in \mathcal{X}\subseteq R^p,\ p>0$ (unstructured) ,
\begin{equation}\label{eq:model}
    Y=f(\bfd) + \epsilon,\ \bfd=\left(\bfs,\mathbf{x}\right),\ \epsilon \overset{\text{iid}}\sim N(0,\sigma^2),
\end{equation}
where $f: \calD\rightarrow R$ is an unknown function defined on the joint input feature space $ \calD\subseteq \mathcal{M} \times \mathcal{X}$, and $\sigma^2$ is an unknown noise variance. 
For example, in a spatial nonparametric regression problem where the response variable $Y$ is the housing price, the structured multivariate features $\bfs$ represent the spatial coordinates, and the unstructured features $\bfx$ consist of other house features. Here, we allow the domain of the structured features to have complex geometries, such as road networks, brain cortical surfaces, and cities with inner lakes. The goal is to predict the response value at new locations as well as understand the relationships between features and response. 
Such problems find wide applications in real estate, public health, social sciences, and environmental studies, to name a few.   
 
Various frequentist and Bayesian methods that employ an ensemble of trees have become popular approaches for estimating nonparametric function $f(\cdot)$. 
Boosting~\citep{freund1997decision}, bagging~\citep{breiman1996bagging}, and random forests~\citep{breiman2001random} are examples of frequentist techniques. The Bayesian Additive Regression Trees (BART) model~\citep{chipman2010bart} has also gained widespread attention because of its exceptional performance and the resulting uncertainty measures. 
The BART model has been applied and extended to various contexts including non-parametric normal response mean regression \citep{chipman2010bart}, classification~\citep{chipman2010bart,zhang2010bayesian,kindo2016multinomial}, variable selection~\citep{chipman2010bart, bleich2014variable,linero2018bayesian2}, estimation of monotone functions~\citep{chipman2021mbart}, causal inference~\citep{hill2011bayesian}, survival analysis~\citep{sparapani2016nonparametric}, heteroscedasticity \citep{bleich2014bayesian, pratola2016efficient}, analysis of log-linear models~\citep{murray2017log} and estimating the intensity of a Poisson process~\citep{lamprinakou2023bart}. Several theoretical studies of BART models~\citep{10.1214/19-AOS1879,  pmlrv89rockova19a, linero2018bayesian} have recently established the posterior convergence rates to provide asymptotic justifications of these methods.

However, the axis-alignment characteristic of the BART model is a major limitation for the scenarios we considered in \eqref{eq:model}: the imposed prior on the trees only allows to split the feature space using one feature at a time,  which can only 
partition the domain into hyper-rectangular regions, restricting the model’s flexibility to account for dependence structures of multivariate features in complex domains when making splits. Numerous frequentist papers have emerged proposing methods for dealing with the axis-alignment effects of decision trees~\citep{NEURIPS2019_043c2ec6,fan2016ostomachion,JMLR:v21:18-664,rainforth2015canonical,rodriguez2006rotation,blaser2021regularizing,blaser2016random}. Bayesian non-parametric methods have also been developed to allow more flexible non-axis-aligned partitioning, including the Ostomachion process~\citep{fan2016ostomachion}, the Random Tessellation Process~\citep{NEURIPS2019_043c2ec6}and the Bayesian Additive Voronoi Tessellations~\citep{stone2024addivortes}. To address this limitation in the BART framework, \citet{MAIA2024107858} incorporate Gaussian process priors for the predictions at each terminal node of the trees.

\citet{luo2022bamdt} introduced the Bayesian additive semi-multivariate decision trees (BAMDT) model to overcome the BART model's axis-alignment limitation. A semi-multivariate decision tree (sMDT) introduced by \cite{luo2022bamdt} divides the joint input feature space $\calD$ into disjoint subsets using multivariate splits for structured features and univariate splits for unstructured features. The multivariate splits provide pliable-shaped partitions that respect the intrinsic geometry and the boundary conditions. However, BAMDT cannot adapt to higher smoothness levels of the unknown regression function of interest, which hinders its performance in many application problems where spatial patterns in some areas are rather smooth.  
Recently, \cite{linero2018bayesian2} proposed an extension of the original axis-parallel BART model, called the soft BART (SBART) model, which introduced a soft split scheme to replace the hard axis-parallel split rule to allow the model to capture a higher degree of smoothness in the unknown regression function. They proved that the posterior distribution of the nonparametric function $f$ concentrates around any unknown true function in the $\alpha$-H{\"o}lder space, for any $\alpha>0$, at near minimax rate.

Motivated by the success of SBART, we develop a new Bayesian additive decision tree model, called SBAMDT, for modeling nonparametric functions in \eqref{eq:model} with both structured and unstructured features. We introduce a novelty decision rule that assumes a mixture of hard and soft decisions at each internal node while allowing the adaptive soft split decision to depend on multivariate features.  
The model is capable of generating a decision boundary at each decision node that either has sharp changes or adaptive smooth transitions. We highlight that our soft decision rule has two major differences from SBART. First, SBART enforces a soft decision rule on every node, while SBAMDT learns a node-specific decision in each tree that permits some internal nodes to make hard decisions. Second, the soft data assignment probability in SBART depends on the distance from a univariate feature value to the cut point. However, such distance is not easily generalized to the multivariate feature manifold case. SBAMDT takes into account both structured and unstructured features, where  To address this challenge, we propose a new approach using distances from each observation to their neighboring reference knot sets to find the probabilities of going left and right at each internal node.


The outline of the paper is as follows. Section~\ref{sec::sMDT} is a review of sMDT used in \cite{luo2022bamdt}. Section~\ref{sec:ssMDT} introduces the adaptive soft Semi-Multivariate Decision Tree, and Section~\ref{sec:inference} presents the proposed inference algorithm. Sections~\ref{sec:SD} and~\ref{sec:NYD} present the application of the algorithm to synthetic data and real data sets, respectively. Section~\ref{sec:conc} provides our conclusions.

\section{Review of Semi-Multivariate Decision Tree}\label{sec::sMDT}
We start by reviewing the notion of a binary tree in order to present the definition of the soft semi-multivariate decision tree. A binary decision tree is a type of tree-data structure that is made up of terminal nodes (also called leaves) and internal nodes. Each internal node $\eta$ is associated with a subset of data, denoted as $D_{\eta}$, and a decision rule that splits $\eta$ into two offspring nodes $\eta_1$ and $\eta_2$. Accordingly, the feature space is split recursively into disjoint subsets, each of which is associated with a terminal node, following the decision path from the root to the leaf.  

Conventionally, the decision rule at an internal tree node has the form $x>c$ or $x<c$ using only one feature $x$ at a time, where $c$ is the cut point in the domain of $x$. \cite{luo2022bamdt} introduced the semi-multivariate decision tree (sMDT) that allows an internal node to split either based on multivariate structured features, $\bfs \in \mathcal{M}$ or a univariate unstructured feature, $x \in \mathcal{X}\subseteq R^p$. 
\citet{luo2022bamdt} modeled the univariate split following the same method as in standard BART introduced by \citet{chipman2010bart}. In a multivariate split, 
\citet{luo2022bamdt} employs a bipartition model via
predictive spanning trees. Specifically, a random spanning tree graph is constructed connecting a set of reference points from the multivariate structured feature space. A graph bipartition model is then developed to split the finite set of reference points. For any arbitrary given data, it is assigned to its nearest reference point in the space of $\calM$. Accordingly, the input feature space is recursively partitioned into subspaces with much more flexible shapes than hyper-rectangulars by standard BART. \citet{luo2022bamdt} illustrated the great utility of this method for spatial regression problems on various domains.  

However, there are two main drawbacks of the method proposed in \citet{luo2022bamdt}. 
First, due to the hard assignment nature of the decision rule, sMDT cannot adapt to the higher smoothness levels of the unknown regression function of interest. Indeed, previous work has shown that the use of piecewise constant basis functions in traditional BART models enables only adaptation to functions in H{\"o}lder space which at most have smoothness $\alpha<1$. 
However, many applications problems involving spatial data may have smoother spatial variations in certain local areas.  Second, \citet{luo2022bamdt} employed the minimum spanning tree (MST) of the k nearest neighbour graph (k-NN) which requires the computation of the geodesic distance between points. Although the geodesic distance has the advantage of respecting manifold structures, it is usually more difficult to compute than the Euclidean distance, which hinders the scalability of sMDT for large-scale data sets.

\section{Adaptive Soft Semi-Multivariate Additive Decision Trees} 
\label{sec:ssMDT}
We propose a new adaptive sMDT with a mixture of hard and soft decision types to address the two limitations of the original sMDT.  The latent regression function $f(\cdot)$ in \eqref{eq:model} is modeled by a sum of $m$ hard-soft Semi-Multivariate Decision Trees,
\begin{equation*}
    f(\bfd)=\sum_{h=1}^m g(\bfd;T^{(h)},\bfM^{(h)}),
\end{equation*}
where $g(\bfd;T^{(h)},\bfM^{(h)})$ denotes the $h$-th hard-soft sMDT weak learner, $T^{(h)}$ denotes the topology of the decision tree, $\bfM^{(h)}=\left(\mu_{h1},\cdots,\mu_{hL_h}\right)$ is a collection of leaf parameters and $L_h$ is the number of leaves.

We begin by introducing the hard partition model for splitting a finite set of reference knots on $\calD$, which shall be used in Section~\ref{subsec:softmodel} to define the mixture of hard and soft split rules to partition the domain $\calD$ and build the probabilistic assignment of observations based on a new distance metric.


\subsection{Hard semi-multivariate decision rule model of reference points}\label{subsec:refmodel}

Let $D^*=\{\left(x_1^*,s_1^*\right),\cdots,\left(x_t^*,s_t^*\right)\}\subseteq \calD $ be a finite set of reference knots on $\calD$. In this study, we draw $D^*$ by randomly sampling a subset of the observed features. We propose a generative prior to recursively split \(D^*\), starting with a root node including the entire \(D^*\). A leaf node is selected for further splitting with probability \(p_{\text{split}}(\eta)\). If a node is chosen to split, a \emph{multivariate} split using the structured features \(\bfs^*\) is performed with probability \(p_m\); otherwise, a \emph{univariate} split using \(x^*\) is performed. 

In a univariate split, $D_{\eta}^*$ is divided into 
\begin{align*}
D_{\eta,1}^*=\{(x^*,\bfs^*)\in D_{\eta}^*: x_{j(\eta)}^* \leq c_{\eta}\}, \quad D_{\eta,2}^*=D_{\eta}^*\backslash D_{\eta,1}^*
\end{align*}
where $x_{j(\eta)}^*$ is the $j_{th}$ structured feature and $c_{\eta}$ the split value of the rule at node $\eta$. For each unstructured feature, we utilize 100 grid points that are equally spaced as potential univariate split cutoff values. After uniformly choosing one of the unstructured features, a uniform selection is made from the available split values associated with that feature. 

In a multivariate split, we assume there is a spanning tree graph, $G_T^*$, whose vertex corresponds to each reference point and edge set encodes the neighborhood structure of the reference points in the manifold multivariate structured feature space $\calM$. Spanning tree enjoys many nice properties for the partition problems of structured data \citep{,luo2022bamdt}.  
Removing one edge from a spanning tree, we easily obtain a contiguous bipartition of $G_T^*$ into two sub-spanning trees whose vertex sets are connected components, which is ideal for obtaining recursive bipartitions of reference points needed in decision tree models. Figure 1(a) shows an example of the semi-multivariate decision tree split that includes both multivariate and univariate splits. Figure 1(b) show an example of a bipartition of $G_T^*$ that connects spatial bins covering Texas, USA, and the resulting two disjoint reference point sets after removing an edge. 

\begin{figure}[H]   
\begin{subfigure}{7 cm}
    \centering\includegraphics[width=6cm]{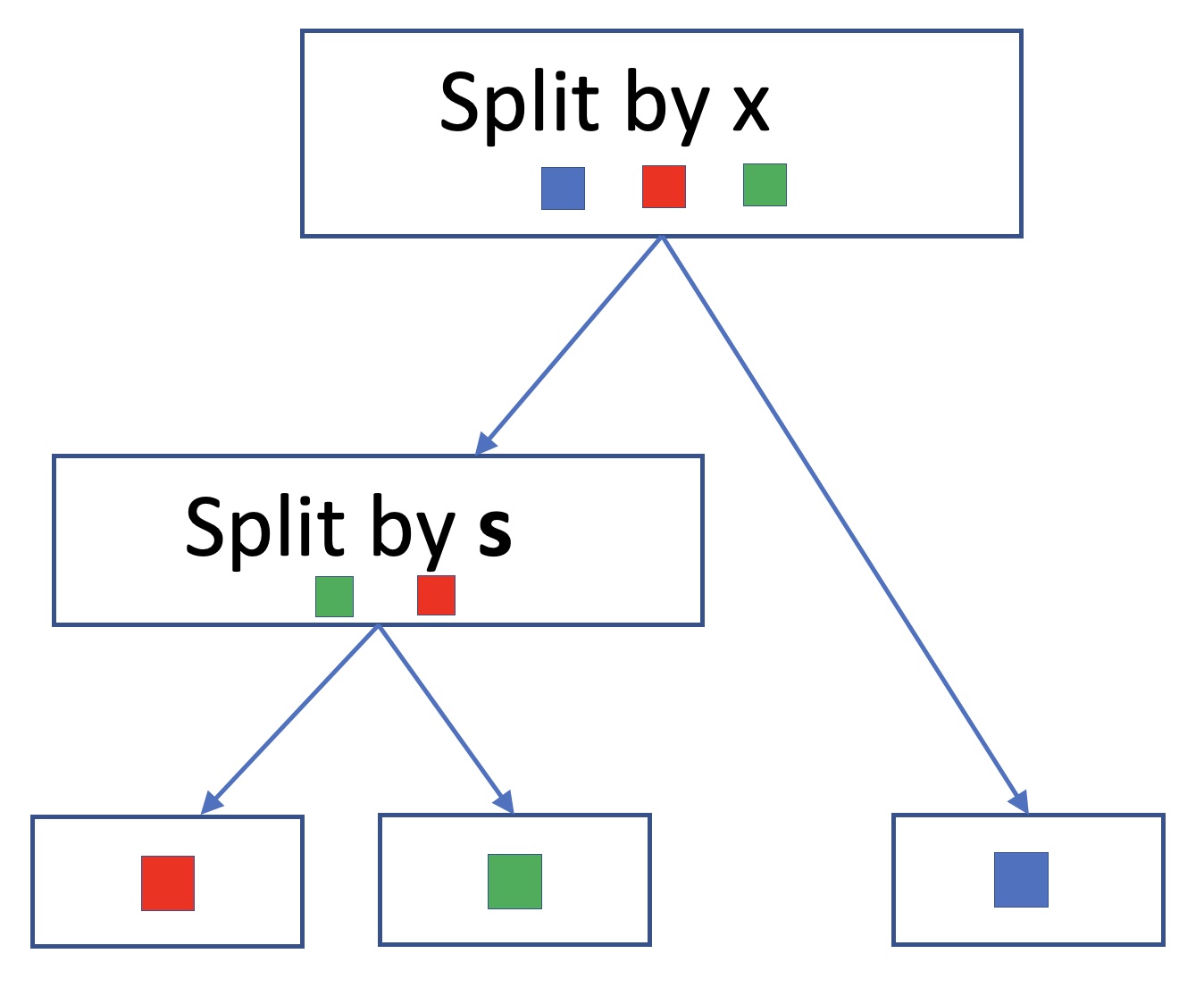}
    \caption{}
  \end{subfigure}
\begin{subfigure}{7 cm}
    \centering\includegraphics[width=6cm]{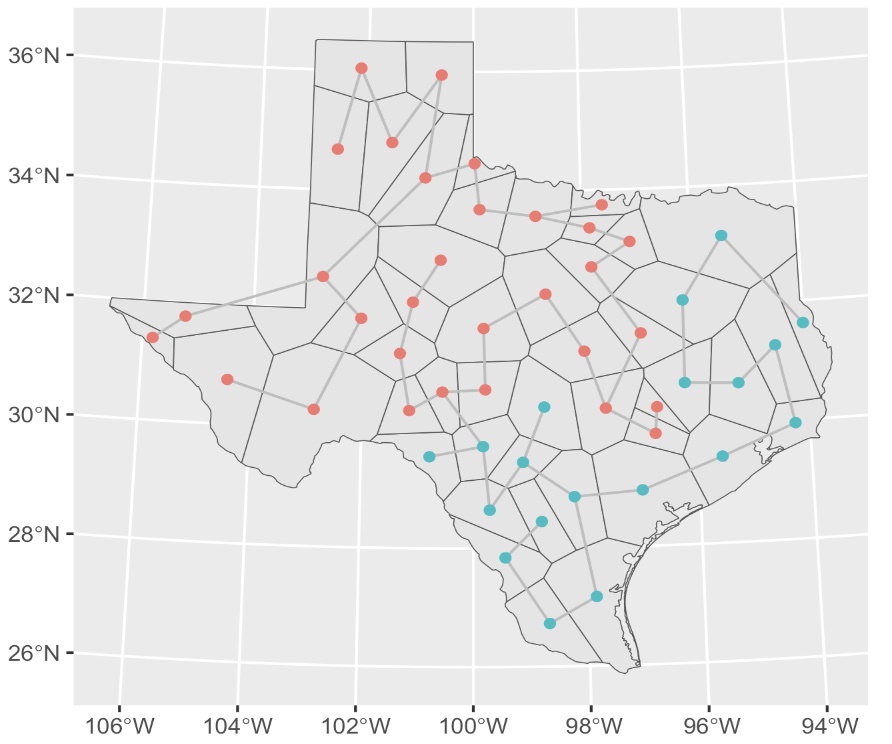}
    \caption{}
 \end{subfigure}
 \caption{(a) An example of a semi-multivariate decision tree; (b) A bipartition of the spanning tree graph  $G_T^*$ into two disjoint reference point sets represented by red and blue colors, respectively. } \label{Fig::demo}
\end{figure}
An important component of the method is the use of the spanning tree. First, we construct a weighted adjacency graph among reference points and then generate its minimum spanning tree, which is defined as the spanning tree minimizing the total edge weight. To address the shortcomings of using geodesic distance-based k-NN graphs as the adjacency graph, we consider the idea of manifold embedding via a graph Laplacian, allowing one to use the ordinary Euclidean distance in the embedded space to approximate distances on the original manifold and construct ~\citep{crane2020survey}. Specifically, we first construct the initial similarity graph by connecting all points on $S^*=\{s_1^*,\cdots,s_t^*\}$ with each other and weighting all edges by their similarity, given by the Gaussian function $f_s(s_i^*,s_j^*)= e^{\left(-||s_i^*-s_j^*||^2\right)}$.  We then compute the normalized graph Laplacian $L$~\citep{shi2000normalized}. Let $U$ be the matrix containing the $k$ eigenvectors of $L$ corresponding to the $k$ smallest non-zero eigen-values, and $\tilde{S}^*=\{u_1^*,\cdots,u_t^*\}$ be the resulting embedded coordinates of the structured reference knots set $S^*$, where $u_i^*$ corresponds to the $i$-th row of the matrix $U$. Embedding via Laplacian graphs has been used in relation to dimensionality reduction, data representation, and diffusion processes~\citep{coifman2006diffusion, belkin2003laplacian,gobel1974random}. \citet{dunson2022graph} recently defined the notion of graph-based Gaussian processes, where the covariance matrix is constructed by such an embedding. 

At each node $\eta$, let $S_{\eta}^*$ and $\tilde{S}_{\eta}^*$ denote the structured reference knots subset on manifold and the embedded structured reference knots subset associated with $\eta$, respectively. The generative bipartition prior for the multivariate split proceeds as follows: we randomly pick two distinct structured reference knots in $\tilde{S}_{\eta}^*$, denoted as $u^*$ and $t^*$. There exists a unique path in $G_T^*$ connecting these two knots. By randomly removing an edge from this path, we obtain two disjoint subsets $\tilde{S}_{\eta,1}^{*}$ and $\tilde{S}_{\eta,2}^{*}$. Equivalently, we also obtain two disjoint subsets $S_{\eta,1}^{*}$ and $S_{\eta,2}^{*}$ in the original manifold space, 
resulting in dividing $D_{\eta}$ into 
\begin{align*}
D_{\eta,1}^*&=\{(x^*,\bfs^*)\in D_{\eta}^*: \bfs^* \in S_{\eta,1}^{*}\} \\
D_{\eta,2}^*&=D_{\eta}^*\backslash D_{\eta,1}^*.
\end{align*}

\subsection{An adaptive Soft semi-multivariate decision rule model}\label{subsec:softmodel} 
Given the recursive bipartition of the reference knots and a data point $\bfd$, we assume that the decision of assigning a data point $\bfd$ to one of the two children nodes is a mixture of hard and soft decisions with different levels of smoothness, and the decision rule under each type of decision is a function depending on the distance of $\bfd$ to its nearest reference knots in the two children nodes. This mixed decision-making allows us to capture both sharp changes and smooth variations in the regression tree function. 

Let  $A_{\eta}^{(h)}$ denote the type of decision at node $\eta$, taking values from $0,1,\ldots, k$. 
We associate a hard decision with node $\eta$ if $A_{\eta}^{(h)}=0$ such that $\bfd$ is assigned to the left or right with probability 1. When $A_{\eta}^{(h)}\in \{1,\ldots,k\}$, a soft decision is adopted at node $\eta$ so that $\bfd$ is either assigned to the left or right child with certain probabilities. We assume that $A_{\eta}^{(h)}$ follows a categorical prior distribution and is independent across internal nodes. Specifically,
$A_{\eta}^{(h)} = \mbox{Categorical}(p_0,p_1,\cdots,p_{k}), \text{ where } p_0+\sum_{i=1}^{k}p_i=1$.

Let $z_{\eta L}^{(h)}(\bfd)|(A_{\eta}^{(h)}=c)$ denote the decision rule function, i.e., the probability of $\bfd$ going left when the $c$-th type of decision is adopted. Let $d_{\eta L}^{(h)}(\bfd)$ and $d_{\eta R}^{(h)}(\bfd)$ denote the distances between $\bfd$ and the nearest knots assigned to the left and right children on the space of features used for splitting $\eta$, respectively (see Sections~\ref{sec::sMDT} and~\ref{subsec:refmodel}). It is reasonable to assume that when $d_{\eta L}^{(h)}(\bfd)<d_{\eta R}^{(h)}(\bfd)$, $\bfd$ is more likely to belong to the left child node. Therefore, we assume  $z_{\eta L}^{(h)}(\bfd)|(A_{\eta}^{(h)}=c)$ is a monotone function of $d_{\eta R}^{(h)}-d_{\eta L}^{(h)}$. 

In the hard case, we assume
\begin{align*}
z_{\eta L}^{(h)}(\bfd) | \left( A_{\eta}^{(h)}=0\right) &=  
\begin{cases}
      1, & \text{if } d_{\eta L}^{(h)}(\bfd) \leq d_{\eta R}^{(h)}(\bfd)  \\
      0, &\text{otherwise} \\
\end{cases} 
\end{align*}


In the soft case,  we introduce a monotone logistic gate function with scaling parameters to transform $d_{\eta R}^{(h)}-d_{\eta L}^{(h)}$ to a value in $[0,1]$ as follows
\begin{eqnarray*}
z_{\eta L}^{(h)}(\bfd)| (A_{\eta}^{(h)}=c)  &=\frac{1}{1 + e^{-\alpha_c^{(h)}\frac{d_{\eta R}^{(h)}(\bfd)-d_{\eta L}^{(h)}(\bfd)}{C_\eta^{(h)}}}} \\
z_{\eta R}^{(h)}(\bfd)|(A_{\eta}^{(h)}=c)  &= 1-z_{\eta L}^{(h)}(\bfd)| (A_{\eta}^{(h)}=c) 
\end{eqnarray*}
where  $C_\eta^{(h)}$ is the maximum of the distances from the nearest left and right knots over all observations introduced to normalize the distances to make them unit-free, and $\alpha_c^{(h)}$ is a decision-specific and tree-specific softness control parameter that scales the normalized distance adaptive to different decisions. The smaller the value of $\alpha$, the smoother the decision boundary. 

We consider two approaches to model the soft case as follows.

\noindent
\textbf{Approach 1.}
In this case, we only consider one soft decision type, i.e., $k=1$, and we shall refer to this approach as S2-BAMDT. We also drop the subscript from $\alpha_1^{(h)}$ for notation simplicity. 
We treat  $\alpha^{(h)}$ as a positive continuous random variable distinct for each tree and 
shall describe its prior model in Section \ref{subsec:priormodel}. It follows that  
\begin{equation*}
    z_{\eta L}^{(h)}(\bfd)|(A_{\eta}^{(h)}=1)   = \frac{1}{1 + e^{-\alpha^{(h)}\frac{d_{\eta R}^{(h)}(\bfd)-d_{\eta L}^{(h)}(\bfd)}{C_\eta^{(h)}}}} . 
\end{equation*}  
This approach requires all the internal nodes of the same tree to have the same smoothness level if their decisions are soft. 
However, we allow each tree to learn its distinct smoothness control parameter $\alpha^{(h)}$ to raise the additive model's flexibility in capturing functions with different levels of smoothness. This way of introducing softness by tree-specific smoothness control parameters resembles that of the soft BART method in \citet{linero2018bayesian}.
However, we emphasize that a key distinction of our approach is the ability to learn a node-specific decision in each tree, allowing some internal nodes to have hard decisions instead of imposing a soft decision rule for all nodes. Another important distinction between our approach and \citet{linero2018bayesian} is that we account for both structured and unstructured features, and we use knot sets at each branch to calculate the distance from each observation and find the probabilities of going left and right at each internal node.

\noindent
\textbf{Approach 2.} In this case, we assume the soft decision at each node is chosen from a set of soft decisions with different levels of smoothness control parameters $\{\alpha_1^{(h)},\ldots, \alpha_{k}^{(h)}\}$, and we refer to this method as Sk-BAMDT. Unlike S2-BADMT, which uses a common smoothness control parameter for the entire tree and relies on varying it to adapt to different smoothness levels, Sk-BAMDT allows for node-specific decisions with varying smoothness control parameters. This enables the model to better adapt to varying smoothness even within a single tree. Moreover, the shallow depth of the trees limits the information available to infer the posterior of \( \alpha_{c}^{(h)} \) especially for higher values of \( k \). Therefore, we recommend fixing $\{\alpha_1^{(h)},\ldots, \alpha_{k}^{(h)}\}$ at a small discrete set of values $\{\alpha_1,\ldots, \alpha_{k}\}$ that are common across all trees. It follows that
\begin{equation*}
    z_{\eta L}^{(h)}(\bfd)|(A_{\eta}^{(h)}=c)   = \frac{1}{1 + e^{-\alpha_c^{(h)}\frac{d_{\eta R}^{(h)}(\bfd)-d_{\eta L}^{(h)}(\bfd)}{C_\eta^{(h)}}}}, \text{ for } c=1,\ldots, k. 
\end{equation*}Compared to Sk-BAMDT, which is sensitive to improperly pre-established smoothness levels, S2-BAMDT is less sensitive to the starting value of the smoothness level per tree.

Figure~\ref{Fig::demo} compares hard and soft  decision trees as derived by~\citet{linero2018bayesian} (SBART) and the hard-soft decision trees of Sk-BAMDT and S2-BAMDT. Sk-BAMDT and S2-BAMDT are two distinct methods for generating hard-soft decision trees that adhere to boundary restrictions and inherent geometry, as seen in figure~\ref{Fig::demo}. In contrast to the single smoothness level applied in S2-BAMDT, the soft boundaries of Sk-BAMDT are determined by two distinct levels of smoothness.




\begin{figure}[H]   
    \centering\includegraphics[width=15cm]{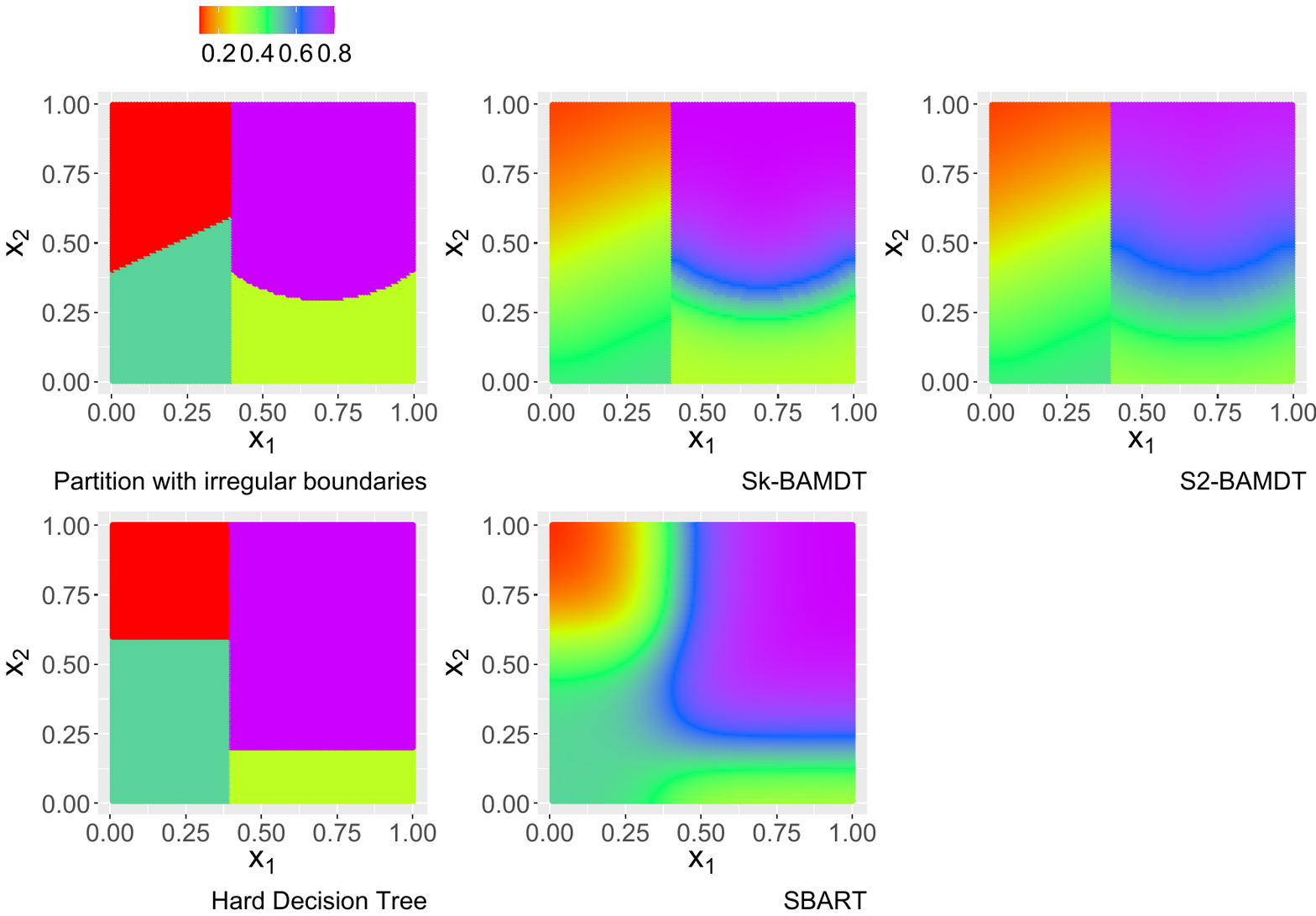}
   
  \caption{Comparison of a hard decision tree to a soft decision tree as derived by Linero (SBART) and the hard-soft decision trees of Sk-BAMDT and S2-BAMDT. We have used a logistic function with a bandwidth parameter equal to 0.08 for SBART, $\{\alpha_1,\alpha_2\}\times q=\{0.5,1\}\times 12$ for Sk-BAMDT, and $\alpha^{(h)}\times q= 0.5\times 12$ for S2-BAMDT. } \label{Fig::Circles}
\end{figure}

Conditional on all the internal node decisions $\bfA^{(h)}\coloneqq \{A_{\eta}^{(h)}\}_{\eta \in T^{(h)}}$ of the $h$-th decision tree $T^{(h)}$, 
the probability of a data point $\bfd$ falling into the terminal node $l$ is given by going down the probablistic decision path of $T^{(h)}$,

\begin{equation}\label{eq:phi}
\bm{\Phi}_{hl}(\bfd)|\bfA^{(h)}=\prod_{\eta\in P_{hl}(\bfd)} [z_{\eta L}^{(h)}(\bfd)| A_{\eta}^{(h)}]
^{\mathds{1}\left(r_{\eta L}^{(h)}=1\right)}\big(1-[z_{\eta L}^{(h)}(\bfd)| A_{\eta}^{(h)}]\big)^{1-\mathds{1}\left(r_{\eta L}^{(h)}=1\right)},
\end{equation}
where $P_{hl}(\bfd)$ is the path from the root to the terminal node $l$, $\mathds{1}(\cdot)$ denotes the indicator function, $r_{\eta L}^{(h)}(\bfd)$ is a binary variable with $r_{\eta L}^{(h)}(\bfd)=1$ if the path goes left at node $\eta$.

Finally, the additive decision trees model takes the form,
\begin{equation}\label{eq:sumtrees}
    f(\bfd)=\sum_{h=1}^m g(\bfd;T^{(h)},\bfM^{(h)})=\sum_{h=1}^{m}\sum_{l=1}^{L_h}\mu_{hl}\bm{\Phi}_{hl}(\bfd),
\end{equation} where the value of each decision tree function is the probability-weighted average of leaf weights.

\subsection{Priors of other model parameters.}\label{subsec:priormodel}
Given a fixed number of trees, $m$, the parameters of the model are the tree components, $(\bfT, \bfM)=\{T^{(h)},\bfM^{(h)}\}_{h=1}^m$, the control parameters $\pmb{\alpha}$, the probabilistic decisions at each internal node, $\bfA=\{\bfA^{(h)}\}_{h=1}^m$, the corresponding 
event probabilities parameters $\bfp_{A}= (p_0,p_1,\ldots,p_k)$ in the categorical distribution for $\bfA$, 
the residual variance, $\sigma^2$, and the leaf parameter variance, $\sigma^2_{\mu}$. 

Following~\citet{chipman2010bart}, we assume that the tree components are independent of each other and that the terminal node parameters of every tree are independent, so that the prior can be factorized as:
\begin{align*}
    P(\bfT,\bfM,\bfA,\pmb{\alpha},\bfp_A, \sigma^2,\sigma_{\mu}^2)
    =& \left(\prod_h P(T^{(h)})\prod_{\eta}P\left(\mu_{h\eta}|T^{(h)} \right) P\left(A_{\eta}^{(h)}|\bfp_A,T^{(h)},\pmb{\alpha}\right) \right)\\& \times P(\pmb{\alpha})P(\bfp_A)P(\sigma^2) P(\sigma^2_{\mu}).
\end{align*} 

\paragraph{Prior on the probabilistic decisions at internal nodes.}
We use a categorical prior distribution to model  the probabilistic decisions $\bfA$ whose event probabilities follows the following prior distribution
\begin{align*}
\bfp_A &\sim \mbox{Dirichlet}(\pmb{\psi}),\ \pmb{\psi}=\left(\psi_0,\psi_1,\cdots, \psi_{k} \right),\ \psi_0,\cdots,\psi_{k} >0,
\end{align*}
We suggest $\psi=\left(1,\cdots,1\right)$ as the default option in Sk-BADMT which works well in our numerical studies. In S2-BAMDT with only two decisions, the categorical distribution is reduced to a Bernoulli distribution, and we assume a beta distribution prior to the probability of having a hard decision. We use $\text{Beta}(1,2)$ as the default choice in all our numerical studies. 
\paragraph{Prior on the softness control parameter.}
 These parameters allow us to appropriately rescale the normalized distance to adapt the probabilities $\{z_{\eta L}^{(h)}(\bfd)| A_{\eta}^{(h)} \neq 0\}$. 
 Sk-BADMT assumes a fixed finite set of $k$ values, $\balpha=\{\alpha_1,\cdots,\alpha_{k}\}$ corresponding to each soft category of $A_{\eta}^{(h)}$. In our numerical examples, we used $\{0.5,1,2\}$ times a positive constant $q\geq 1$ for $\balpha$ while tuning the choice of $q$ in different studies, which seems to work well in practice. Within the range of values between 1 and 13, we choose those values of the
tuning parameter $q$ that give the most accurate predictions based on Watanable-Akaike information criterion~\citep{gelman2014understanding} and cross-validation~\citep{gelman2014understanding} for each simulation scenario.  
In S2-BADMT, we use a gamma distribution with shape $\alpha_g$ and rate $\beta_g$ to model the control parameter, $ \alpha^{(h)}\sim \mbox{Gamma}(\alpha_g,\beta_g)$. 
We set $\alpha_g=1$ and $\beta_g=0.5$ in our numerical studies.

\paragraph{Prior on the leaf weight parameters}
After rescaling $Y$ into [-0.5, 0.5], we use the conjugate normal distribution for $\mu_{hl}$, conditional on $T^{(h)}$ :
\begin{align*}
    \mu_{hl}|T^{(h)},\sigma^2_{\mu} &\sim \mbox{N}\left(0, \sigma^2_{\mu}\right),  \quad \sigma^2_{\mu} \sim \mbox{Inv-Gamma}(\alpha_\mu,\beta_\mu). 
\end{align*}We apply an Inverse-Gamma prior to the variance of the leaf weight parameter, $\sigma_{\mu}^2$, following~\citet{he2023stochastic}, with hyperparameters $\alpha_{\mu} = 3$ and $\beta_{\mu} =0.5\times\mathrm{Var}(y)/m$.
\paragraph{Prior on the residual variance}
In line with the original BART model~\citep{chipman2010bart}, we use a conjugate inverse-$\chi^2$ prior for the residual variance, that is, $\sigma^2 \sim v\lambda/\chi^2_v$ for $v=3$, and select $\lambda$, such that $P(\sigma^2<\hat{\sigma}^2)=0.90$ a priori, where $\hat{\sigma}^2$ is the sample variance of the responses. 

\paragraph{Prior on the decision trees.}
Following \citet{chipman2010bart}, we adopt the Galton-Watson process~\citep{harris1963theory} to model the tree prior for the recursive hard bipartitions of reference points.  Each node has either zero or two offspring. The probability of splitting a node depends on its depth on the tree, taking the form:
\begin{equation*}
    p_{split}\left(\eta\right)=\frac{\gamma}{\left(1+d(\eta)\right)^{\delta}},
\end{equation*} where $d(\eta)$ is the depth of node $\eta$ in the tree. The hyper-parameters $(\gamma, \delta)$ control the shape of trees. The parameter $\gamma>0$ controls the probability that the root of a tree will split into two offspring, while the parameter $\delta>0$ penalizes against deep trees. As noted in~\cite{chipman2010bart}, we want to keep the depth of the tree small whilst ensuring non-trivial trees. In our simulation study on real and synthetic data, we fix $\gamma=0.95$ and $\delta=2$. 

If the node splits, we perform a multivariate split with probability $p_m$. Otherwise, we perform a univariate split. To perform a univariate split of $D_{\eta}^{*}$, we select uniformly one of the unstructured features, followed by a uniform selection from the available split values associated with that feature. We set $p_m=d_{\mathcal{M}}/(d_{\mathcal{M}}+p)$ where $d_{\mathcal{M}}$ is the dimensionality of the structured feature space $\mathcal{M}$ and $p$ is the number of unstructured features.


\subsection{Connection between GP and the SBAMDT model}

In this subsection, we explore the connection between SBAMDT and Gaussian processes (GPs) to reveal the behavior of decision tree models in capturing functional dependence and smoothness. As previously described, BART represents the latent regression function $f$ as a summation of $m$ decision tree piecewise constant functions. In this framework, the value produced by each tree is determined through a probability-weighted average of leaf weights, which depends on the decisions made at the internal nodes:
\[
f(\mathbf{d}) = \sum_{h=1}^m g(\mathbf{d}; T^{(h)}, \mathbf{M}^{(h)}) = \sum_{h=1}^{m} \sum_{l=1}^{L_h} \mu_{hl} \phi_{hl}(\mathbf{d}),
\]

where \( \phi_{hl}(\mathbf{d}) \) is defined as:

\[
\phi_{hl}(\mathbf{d}) | \mathbf{A}^{(h)} = \prod_{\eta \in P_{hl}(\mathbf{d})} P_{\eta}(\mathbf{d}; A_{\eta}^{(h)}),
\]

with the conditional probability given by:

\[
P_{\eta}(\mathbf{d}; A_{\eta}^{(h)}) = [z_{\eta L}^{(h)}(\mathbf{d}) | A_{\eta}^{(h)}]^{\mathds{1}(r_{\eta L}^{(h)} = 1)} \left(1 - [z_{\eta L}^{(h)}(\mathbf{d}) | A_{\eta}^{(h)}]\right)^{1 - \mathds{1}(r_{\eta L}^{(h)} = 1)}.
\]

Under the normal prior and posterior distributions for the leaf weight parameters, we present the following theorem, divided into two parts: the prior distribution and the posterior distribution.

\newpage

\begin{theorem}[Connection between SBAMDT and GP]  \renewcommand\labelenumi{(\theenumi)} \label{Th1}
\item \textbf{Prior distribution:}
\begin{enumerate}
\item Conditional on $\bfT$ and $\bfA$, the prior distribution of $f$ is given by $f\sim GP\left(0,C(f_i,f_j)\right)$, where the covariance function is determined by $C(f_i,f_j)=\frac{\beta_{\mu}}{\alpha_{\mu}-1} \sum_{h=1}^m\sum_{l=1}^{L_h} \phi_{hl}(\bfd_i)\phi_{hl}(\bfd_j)$, and $\alpha_{\mu}$ and $\beta_{\mu}$ are the hyperparameters of the Inverse-Gamma prior applied to the variance of the leaf-weight parameter. In the absence of hard splits, the process is mean square differentiable. However, if there is at least one hard split due to discontinuities in the indicator function, we cannot assert that the process is mean square differentiable overall.
\item Conditional on $\bfT$, the prior distribution of $f$ is  $f\sim GP\left(0,C(f_i,f_j)\right)$ with the covariance matrix defined by $C(f_i,f_j)=\frac{\beta_{\mu}}{\alpha_{\mu}-1} \sum_{h=1}^m\sum_{l=1}^{L_h}E\left(\phi_{hl}(\bfd_i)\phi_{hl}(\bfd_j)\right)$, where the expectation is with respect to the distribution of $\bfA$, which holds in the limit as the number of trees approaches infinity. If there is at least one hard split, the process is not mean square  differentiable everywhere.

\end{enumerate}

\item \textbf{Posterior Distribution:}
\begin{enumerate}
\item We denote with $\{\hat{\bfT},\hat{\bfM},\hat{\bfA},\hat{\sigma}^2,\hat{\sigma}_{\mu}^2\}$ the posterior sample of $\{\bfT,\bfM,\bfA,\sigma^2,\sigma_{\mu}^2\}$.\\ Given $\{\hat{\bfT},\hat{\bfM}^{(-h)},\hat{\bfA},\hat{\sigma}^2,\hat{\sigma}_{\mu}^2\}$, the distribution of the latent regression function $f(\bfd)$ is a valid Gaussian process with a mean given by $\hat{\mu}_f^{(h)}(\bfd)= \hat{\mu}^{(h)T}\hat\Phi^{(h)}(\bfd)+\sum_{h'=1,h\neq h'}^m\hat{\bfM}^{(h')T} \hat{\bm{\Phi}}^{(h')}(\bfd)$ and a covariance function given by 
$\hat{C}(f_i,f_j)=\hat{\bf{\Phi}}^{(h)T}(\bfd_i)\Omega^{(h)}\hat{\bf{\Phi}}^{(h)}(\bfd_j)$, where
$\hat{\bm{\Phi}}^{(h)}(\bfd)=\left(\phi_{h1}(\bfd),\ldots,\phi_{hL_h}(\bfd)\right)^T$, where $\hat{\bm\mu}^{(h)}$ and $\Omega^{(h)}$ are the mean and covariance matrix of the conditional distribution of $M^{(h)}$ that is a multivariate Gaussian distribution.

\item We denote with $\{\hat{\bfT},\hat{\bfM},\hat{\bfA},\hat{\sigma}^2,\hat{\sigma}_{\mu}^2\}$ the posterior sample of $\{\bfT,\bfM,\bfA,\sigma^2,\sigma_{\mu}^2\}$. Given $\{\hat{\bfT},\hat{\bfA},\hat{\sigma}^2,\hat{\sigma}_{\mu}^2\}$, the posterior distribution of the latent regression function $f(\bfd)$ is a valid Gaussian process with a mean given by $\hat{\mu}_f(\bfd)=\sum_{h=1}^m E\left(M^{(h)}\right)^{T}\hat{\bm{\Phi}}^{(h)}(\bfd)$ and a covariance function given by\\
$\hat{C}(f_i,f_j)=\sum_{h=1}^m\sum_{h'=1}^m\hat{\bf{\Phi}}^{(h)T}(\bfd_i)C\left(M^{(h)}, M^{(h')}\right)\hat{\bf{\Phi}}^{(h')}(\bfd_j) $, where
$\hat{\bm{\Phi}}^{(h)}(\bfd)=\left(\phi_{h1}(\bfd),\ldots,\phi_{hL_h}(\bfd)\right)^T$.
\end{enumerate}

\end{theorem}
\begin{proof}
Appendix~\ref{Theorem1} includes the proof of Theorem~\ref{Th1}.
\end{proof} 

Building on Theorem~\ref{Th1}, we further explore a specific hierarchical prior construction of decision trees based on the number of leaves, tree structures, and decision rules.
The following proposition outlines the behavior of the resulting stochastic process of the latent regression function.  

\begin{proposition}\label{Prop1}
Consider a hierarchical prior in which a shifted Poisson distribution is placed on the number of leaves, $P(L_h=k)=\frac{e^{-\lambda}\lambda^{k-1}}{(k-1)!},\ k=1,2,3,\ldots$, and conditional on $L_h$, a uniform distribution on the tree structures with $L_h$ leaves. We denote with $R^{(h)}$ the set of decision rules in $T^{(h)}$. Conditional on $T^{(h)}$ and $L_h$, the distribution of $R^{(h)}$ is given by the product rule probabilities,  $P(R^{(h)}|T^{(h)},L_h)=\prod_{\eta \in T^{(h)}}p_{rule}(\eta)$, where $p_{rule}(\eta)$ is the probability of the decision rule for node $\eta$ . Then, $f\sim GP(0, C\left(f_i,f_j\right))$ with \\$C\left(f_i,f_j\right)=
\frac{\beta_{\mu}}{\alpha_{\mu}-1} \sum_{h=1}^m \sum_{k=1}^{\infty} P\left(L_h=k\right)P\left(T^{(h)}|L_{h}=k\right)P(R^{(h)}|T^{(h)})\sum_{l=1}^{k}E\left(\phi_{hl}(\bfd_i)\phi_{hl}(\bfd_j)\right)$ as $m\rightarrow \infty$.    
\end{proposition}

With the foundation established in Theorem~\ref{Th1} and the subsequent Proposition~\ref{Prop1}, we can now contrast our SBAMDT model with the XBART-GP strategy developed by \citet{wang2024local}. While XBART-GP fits a Gaussian process at each leaf node of a BART tree and utilizes a squared exponential kernel to define the covariance function achieve a smooth stochastic process for prediction, our SBAMDT model naturally provides a globally valid stochastic process and hence a coherent framework for model estimation and prediction. SBAMDT does not require strong parametric assumptions on the covariance matrix. Instead, it quantifies the relative similarity between data points by multiplying their probabilities of falling into the same leaf and summing these products across all tree structures. The posterior covariance matrix of $f$ involves the posterior of decision trees, making it nonstationary and highly adaptive. This allows SBAMDT to capture the complex relationships in the data more effectively than the XBART-GP method, which relies on predefined distance metrics.

\section{Inference algorithm}
\label{sec:inference}
We summarize our Bayesian hierarchical model below 
\begin{align*}
    Y &=f(\bfd)+\epsilon,\ 
\epsilon \overset{\text{iid}}\sim \mbox{N}\left(0,\sigma^2\right),\ \bfd=\left(\bfs,\mathbf{x}\right)  \\
f(\bfd)&=\sum_{h=1}^m\sum_{l=1}^{L_h}\mu_{hl}\phi_{hl}(\bfd) \\
\phi_{hl}(\bfd)|\bfA^{(h)} &=\prod_{\eta\in P_{hl}(\bfd)} [z_{\eta L}^{(h)}(\bfd)| A_{\eta}^{(h)}]
^{\mathds{1}\left(r_{\eta L}^{(h)}=1\right)}\big(1-[z_{\eta L}^{(h)}(\bfd)| A_{\eta}^{(h)}]\big)^{1-\mathds{1}\left(r_{\eta L}^{(h)}=1\right)}\\
\mu_{hl}| (T^{(h)},\sigma^2_{\mu}) &\overset{\text{iid}}\sim \mbox{N}\left(0,\sigma_{\mu}^2 \right),\quad 
\sigma^2_{\mu}\sim \mbox{Inv-Gamma}(\alpha_\mu,\beta_\mu), \quad 
\sigma^2  \sim \frac{v\lambda}{\chi_v^2}, \\
T^{(h)} &\sim \text{Galton-Watson process } \\
\text{\textbf{Approach 1:}} \\
A_{\eta}^{(h)}|\bfp_A &\sim \mbox{Multinomial}(\bfp_A),\ \bfp_A=\left(p_0, p_2,\cdots,p_{k}\right),\ \sum_{i=1}^{k}p_i=1-p_0\\
\bfp_A&\sim \mbox{Dirichlet}(\pmb{\psi}),\ \pmb{\psi}=\left(\psi_0,\cdots, \psi_{k} \right),\ \psi_0,\cdots,\psi_{k} >0.  \\
\text{\textbf{Approach 2:}} \\
a^{(h)} &\sim \mbox{Gamma}\left(a_g,\beta_g \right) \\
A_{\eta}^{(h)}|p_A &= \begin{cases}
1\ (\text{hard}) & \text {wp } p_A \\
0\ (\text{soft})& \text {wp } 1-p_A
\end{cases}\\
p_A&\sim \mbox{Beta}(s_a,s_b).
\end{align*} 
 We describe the Bayesian inference algorithm for Sk-BAMDT and defer the detailed algorithm for S2-BAMDT to Appendix~\ref{sec::S2BAMDT}. Given the observed data $\bfY=\{Y_i\}_{i=1}^n$, we seek to draw posterior samples of the model parameters $\left(\bfT,\bfM,\bfA,\bfp_A,\sigma^2,\sigma_{\mu}^2\right)$ from their posterior distributions. 

For any arbitrary tree $T^{(h)}$, let us denote the set of the remaining trees as $T^{(-h)}=\{T_j\}_{j \neq h}$ whose associated leaf parameters are denoted as $\bfM^{(-h)}$ and node probabilistic decisions are denoted as $\bfA^{(-h)}$. Let $\bfR^{(h)}=\bfY-\sum_{j\neq h}g(\bfd; T^{(j)}, \bfM^{(j)})$ denote the residual response data conditional on other decision trees.  
We follow the backfitting Markov chain Monte Carlo (MCMC) sampler \citep{hastie2000bayesian} to draw each decision tree parameters $(T^{(h)},\bfM^{(h)},\bfA^{(h)})$ from their respective full conditional distributions of global model parameters $(\bfp_A,\sigma^2,\sigma_{\mu}^2)$ and other local decision tree specific parameters, $(T^{(-h)},\bfM^{(-h)},\bfA^{(-h)})$ as shown in Algorithm~\ref{alg:MH}. 


A draw from $T^{(h)},\bfM^{(h)},\bfA^{(h)}|T^{(-h)}, \bfM^{(-h)}, \bfA^{(-h)}, \sigma^2, \sigma_{\mu}^2,\bfp_A, Y$ is equivalent to a draw from $T^{(h)},\bfM^{(h)},\bfA^{(h)}|\bfR^{(h)},\sigma^2,\sigma_{\mu}^2,\bfp_A$. 
A draw from $T^{(h)},\bfM^{(h)},\bfA^{(h)}|\bfR^{(h)},\sigma^2,\sigma_{\mu}^2,\bfp_A$ is equivalent to a draw from $T^{(h)}|\bfR^{(h)},\sigma^2,\sigma_{\mu}^2,\bfp_A,Y$ followed by a draw from $\bfA^{(h)}|T^{(h)},\bfR^{(h)},\sigma^2,\sigma_{\mu}^2,\bfp_A$ and a draw from $\bfM^{(h)}|T^{(h)},\bfA^{(h)},\bfR^{(h)},\sigma^2,\sigma_{\mu}^2$.  

We propose a Metropolis-Hastings Algorithm to sample from 
\begin{equation*}
    P\left(T^{(h)}|\bfR^{(h)},\bfA^{(h)},\sigma^2,\sigma^2_{\mu}\right) \propto P\left(\bfR^{(h)}|T^{(h)},\bfA^{(h)},\sigma^2,\sigma^2_{\mu}\right)P(T^{(h)}).
\end{equation*}
The transition kernel is chosen among the three proposals: GROW, PRUNE, and CHANGE. The GROW proposal randomly picks a terminal node, splits the chosen terminal into two new nodes, and assigns a decision rule to it. The PRUNE proposal randomly picks a parent of two terminal nodes and turns it into a terminal node by collapsing the nodes below it. The CHANGE proposal randomly picks an internal node and randomly reassigns to it a probabilistic decision. We describe the implementation of the proposals in Appendix~\ref{Ap::MH_Trees}. In our simulation study, the probabilities of the proposals are set to: P(GROW) = P(PRUNE) = 0.4, and P(CHANGE) = 0.2.

Then, we have the following results required for the implementation of Algorithm~\ref{alg:MH}.
\begin{enumerate}
    \item The conditional likelihood is given by 
    \begin{equation}\label{eq:Lik}
    P(\bfR^{(h)}|\bfM^{(h)},T^{(h)},\bfA^{(h)},\sigma^2)=\prod_{i=1}^n (2\pi\sigma^2)^{-1/2} \exp\left(-\frac{1}{2\sigma^2}\left(R_i^{(h)}-\sum_{l=1}^{L_h}\mu_{hl}\phi_{hl}(\bfd_i)\right)^2\right).
    \end{equation}
    \item The conditional integrated likelihood is given by 
    \begin{equation} \label{eq:IntLik}
        P\left(\bfR^{(h)}|T^{(h)},\bfA^{(h)},\sigma^2,\sigma_{\mu}^2\right)=\frac{|2\pi\Omega|^{1/2}}{(2\pi\sigma^2)^{n/2}|2\pi\sigma^2I|^{1/2}}\exp\left(-\frac{||\bfR^{(h)}||^2}{2\sigma^2}+\frac{1}{2}\hat{\mu}^T\Omega^{-1}\hat{\mu}\right),
    \end{equation} where 
    \begin{align*}
        \hat{\mu}&=\Omega \sum_{i=1}^n\frac{R_i^{(h)}\bm{\Phi}_i}{\sigma^2},\ \Omega^{-1}=\left(\frac{I}{\sigma_{\mu}^2}+\Lambda\right),\
        \Lambda =\frac{1}{\sigma^2}\sum_{i=1}^n\bm{\Phi}_i\bm{\Phi}_i^T,\
        \bm{\Phi}_i^T=\left(\phi_{h1}(\bfd_i),\cdots,\phi_{hL_h}(\bfd_i)\right),
    \end{align*} $I$ the $L_h\times L_h$ identity matrix, $|U|$ the determinant of the matrix $U$ and $||z||$ the Euclidean norm of vector z.
    \item The conditional distribution of $\bfM^{(h)}$ is a multivariate Gaussian distribution with mean $\hat{\bm\mu}^{(h)}$ and covariance matrix $\Omega^{(h)}$:
    \begin{equation*}
        \bfM^{(h)}|\bfR^{(h)}, T^{(h)},\bfA^{(h)},\sigma^2,\sigma_{\mu}^2 \sim \mbox{N}\left(\hat{\bm\mu}^{(h)},\Omega^{(h)}\right).
    \end{equation*}
    \item The conditional distribution of $\sigma^2$ is an Inverse-gamma distribution with shape $s_{1\sigma}=\frac{n+v}{2}$ and scale $s_{2\sigma}=\frac{1}{2}\sum_{i=1}^n \left( Y_i -\sum_{h=1}^m\sum_{l=1}^{L_h}\mu_{hl}\phi_{hl}(\bfd_i) \right)^2 + \frac{v\lambda}{2}$.

    \item The conditional distribution of $\sigma_{\mu}^2$ is an Inverse-gamma distribution with shape $s_{1\mu}=\alpha_{\mu} + \frac{1}{2}\sum_{h=1}^m L_h$ and scale $s_{2\mu}=\frac{1}{2}\sum_{h=1}^m\sum_{l=1}^{L_h} \mu_{hl}^2 + \beta_{\mu}$. 
    \item The conditional distribution of $\bfp_A$ is a Dirichlet distribution:
     \begin{equation*}
         P(\bfp_A|\bfT,\bfA) \sim \mbox{Dirichlet}(\tilde{\psi}),\ 
        \tilde{\pmb{\psi}}=(\tilde{\psi}_0,\cdots,\tilde{\psi}_{k}),\ 
        \tilde{\psi}_l = \sum_{h} \sum_{\eta} \mathds{1}\left(A_{\eta}^{(h)}=l\right) +\psi_l. 
     \end{equation*}
    \item The conditional distribution of $A_{\eta}^{(h)}$ is a Multinomial distribution : 
    \begin{align*}
        A_{\eta}^{(h)}|T^{(h)},\bfp_A,\bfR^{(h)},\sigma^2,\sigma_{\mu}^2 &\sim \mbox{Multinomial} (w_{Ah}),\  
        w_{Ah}=(w_{h0},\cdots,w_{h{k}}) \\
        w_{hl} &= \frac{p_l P\left(\bfR^{(h)}|T^{(h)},A_{\eta}^{(h)}=l,\sigma^2,\sigma_{\mu}^2\right)}{\sum_{j=0}^kp_j P\left(\bfR^{(h)}|T^{(h)},A_{\eta}^{(h)}=j,\sigma^2,\sigma_{\mu}^2\right)}.
    \end{align*}

\end{enumerate}
The proof can be found in the appendix.

\begin{algorithm}[H] 
\caption{Metropolis-Hastings within Gibbs sampler (Sk-BAMDT)} 
\begin{algorithmic}
\label{alg:MH}
\FOR{$t=1,2,3,..$ }
\FOR{$h=1$ to $m$ }
\STATE{Sample $T^{(h)}|\bfR^{(h)},\sigma^{2},\sigma_{\mu}^{2},\bfp_A$ using a Metropolis-Hastings Algorithm.}
\STATE{If $T^{(h)}$ grows, sample the probabilistic decision for the new internal node $\eta$, $A_{\eta}^{(h)}|\bfR^{(h)},T^{(h)},\sigma^2,\sigma_{\mu}^2,\bfp_A$, from $\mbox{Multinomial}(w_{Ah})$.}
\STATE{Sample $ \bfM^{(h)}|\bfR^{(h)}, T^{(h)},\bfA^{(h)},\sigma^2,\sigma_{\mu}^2$ from $\mbox{N}\left(\hat{\pmb\mu}^{(h)},\Omega^{(h)}\right)$. }
\ENDFOR
\STATE{Sample $\sigma^{2}|\bfT,\bfM,\bfA,\bfY$ from $\mbox{Inverse-Gamma}(s_{1\sigma},s_{2\sigma})$. }
\STATE{Sample $\sigma^{2}_{\mu}|\bfT,\bfM,\bfA,\bfY$ from $\mbox{Inverse-Gamma}(s_{1\mu},s_{2\mu})$. }
\STATE{Sample $\bfp_A|\bfT,\bfA$ from $\mbox{Dirichlet}(\tilde{\pmb{\psi}})$.}
\ENDFOR
\end{algorithmic}
\end{algorithm}

\section{Simulation study on synthetic data}
\label{sec:SD}
Using synthetic data, we show the effectiveness of our proposed SBAMDT models, namely Sk-BAMDT and S2-BAMDT.  In addition to SBAMDT, we consider three benchmark additive decision tree models for comparisons, including BART~\citep{chipman2010bart}, SBART~\citep{linero2018bayesian}, and BAMDT~\citep{luo2022bamdt}. 
We present two simulation scenarios in this section. The first has a rotated U-shaped domain with circular boundaries, where SBAMDT and BAMDT are particularly suitable. The second has a square domain with horizontal and vertical boundaries to examine the robustness of our method in a case that favors the competing methods, BART and SBART.  See Appendix 2 for an additional simulation example where the true function is a piecewise Gaussian process on the U-shaped domain.

For each model, we employ $m = 30$ weak learners. 
For each unstructured feature, we utilize 100 equally spaced grid points as candidate split cutoff values. We discard the first 5000 iterations as burn-in and apply thinning by retaining one sample every five iterations from the remaining 15000 iterations of the MCMC algorithms in SBAMDT, BAMDT and BART. For SBART, we discard the first 8000 iterations as burn-in and save 3000 samples after thinning. 
 
We evaluate the prediction performance of SBAMDT and its competitors in terms of root mean square prediction error (RMSPE), mean absolute prediction error (MAPE), and continuous ranked probability score \citep[CRPS,][]{gneiting2007strictly}. For all metrics, lower values indicate better performance. 

\subsection{U-shape Example}
We consider a 45-degree rotation of a U-shaped domain using longitude and latitude as structural features. A 0.9-radius circle centred at the origin divides the U-shaped domain into three clusters. We generate ten unstructured features where each feature is an independent realisation from a Gaussian process on the U-shape domain. We use a training data set of size $n = 500$, a test data set of size $n_{test} = 200$, and a piecewise smooth function dependent on $(\bfs,x_1)$ and their interaction. See Figure \ref{Fig::ScU} for an illustration of the true function and the domain. With a noise level of $\sigma = 0.1$, we generate 50 response replicates in the training and test data from Equation~\eqref{eq:model}.  A random subset of 100 training data is employed as knots for BAMDT and SBAMDT. 


\begin{figure}[h]   

    \centering\includegraphics[width=15cm,height=8cm]{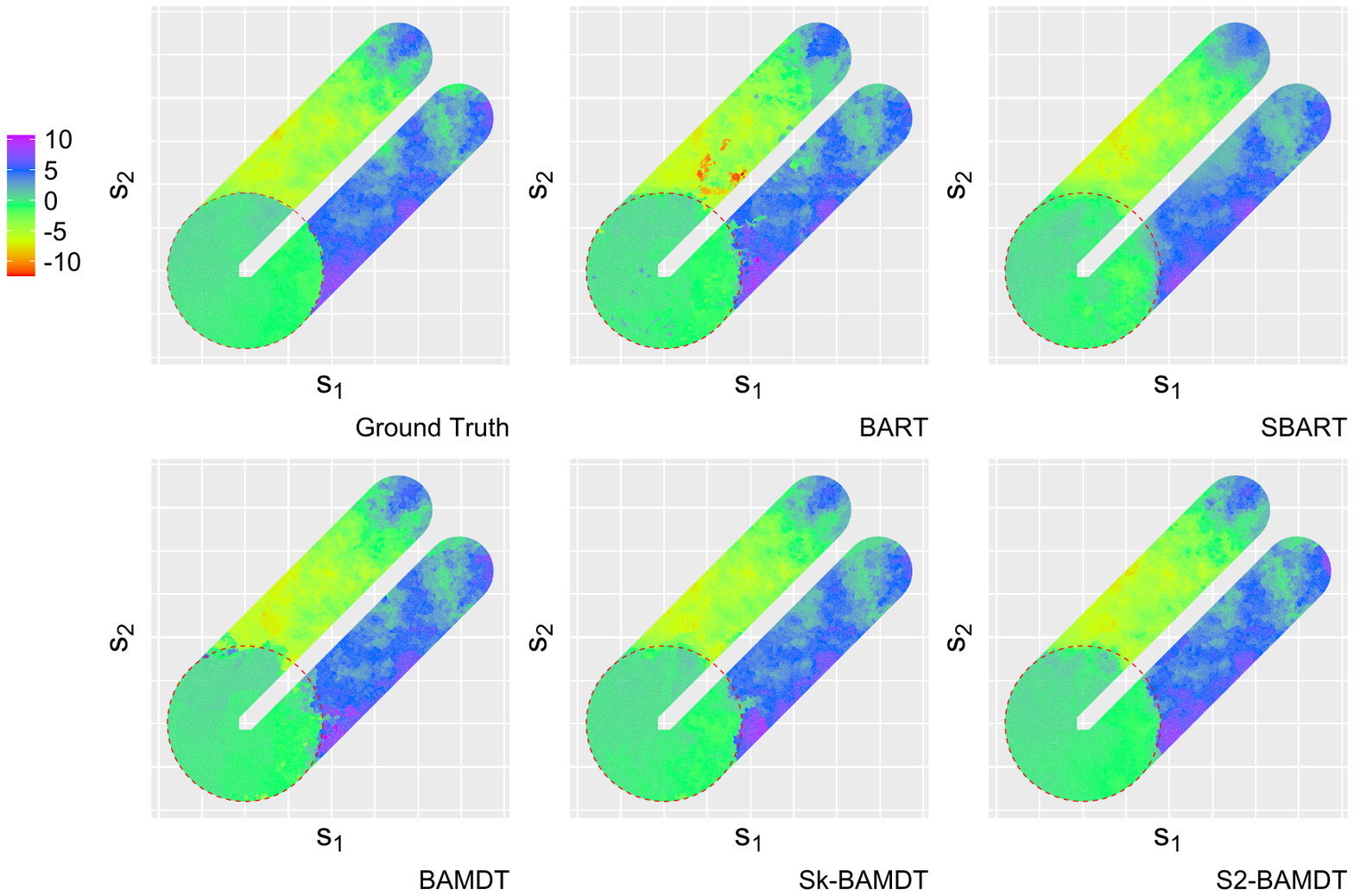}
  \caption{ The ground truth for $f(\bfs,\mathbf{x})$ and the predictive surfaces $\hat{f}(\bfs,\mathbf{x})$ of each method for a U-shape simulated data. The red circle indicates discontinuity boundaries in the true function projected to the 2-D U-shape domain.} \label{Fig::ScU}
\end{figure}

\begin{table}[H]
 \caption{\textbf{U-shape.} The average performance metrics and their standard deviations (in parenthesis) for SBAMDT and benchmark models over 50 replicates. 
 } 
 \label{tab:PM_ScU} 
\begin{tabular}{p{2.5 cm} p{1.8 cm} p{1.8 cm} p{1.8cm}  p{2.5 cm} p{2.5 cm}   }
 \hline
 \multicolumn{6}{c}{ SBAMDT and Benchmark models}  \\
 \hline
   & BAMDT  &BART  & SBART & Sk-BAMDT (q=8) & S2-BAMDT (q=8)  \\ \hline
MAPE$\times10^{-1}$& 4.88(0.53)  & 7.09(0.48) & 6.36(0.29)  & \textbf{3.93(0.43)} & 4.13(0.72)  \\ 
RMSPE$\times10^{-1}$& 10.01(1.29)& 12.52(1.00) & 10.18(0.38)  & \textbf{9.14(0.77)}& 9.73(1.22)   \\
CRPS$\times10^{-1}$& 4.00(0.50)& 5.85(0.45)& 4.70(0.19)& \textbf{3.03(0.32)} & 3.24(0.47)  \\
 \hline
\end{tabular}
\end{table}

\begin{figure}[h]   
    \centering\includegraphics[width=15cm,height=8cm]{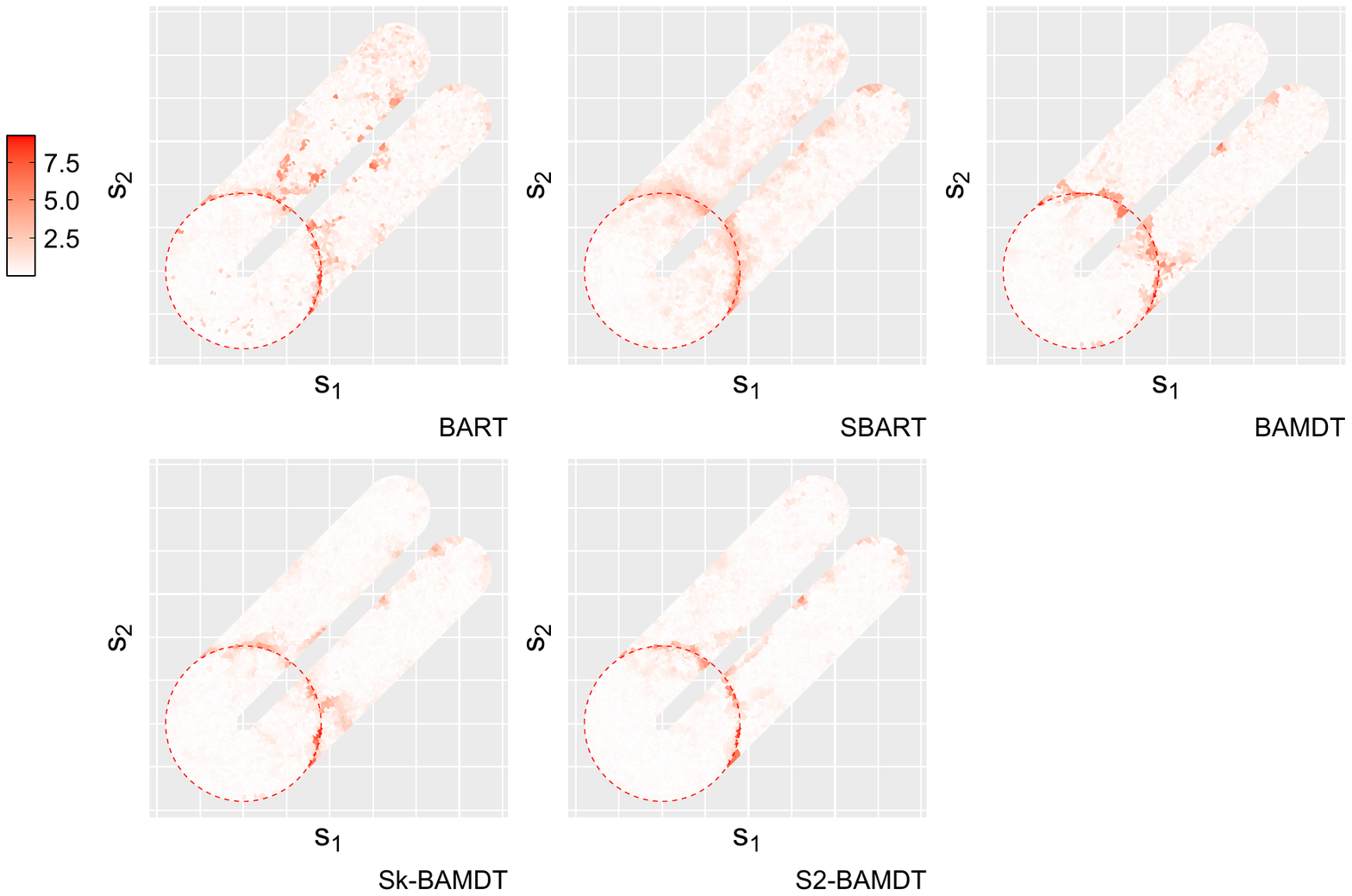}
  \caption{ The APE of each method for one U-shape simulation.} \label{Fig::MAPE_ScU}
\end{figure}

\begin{figure}[h]   
    \centering\includegraphics[width=15cm]{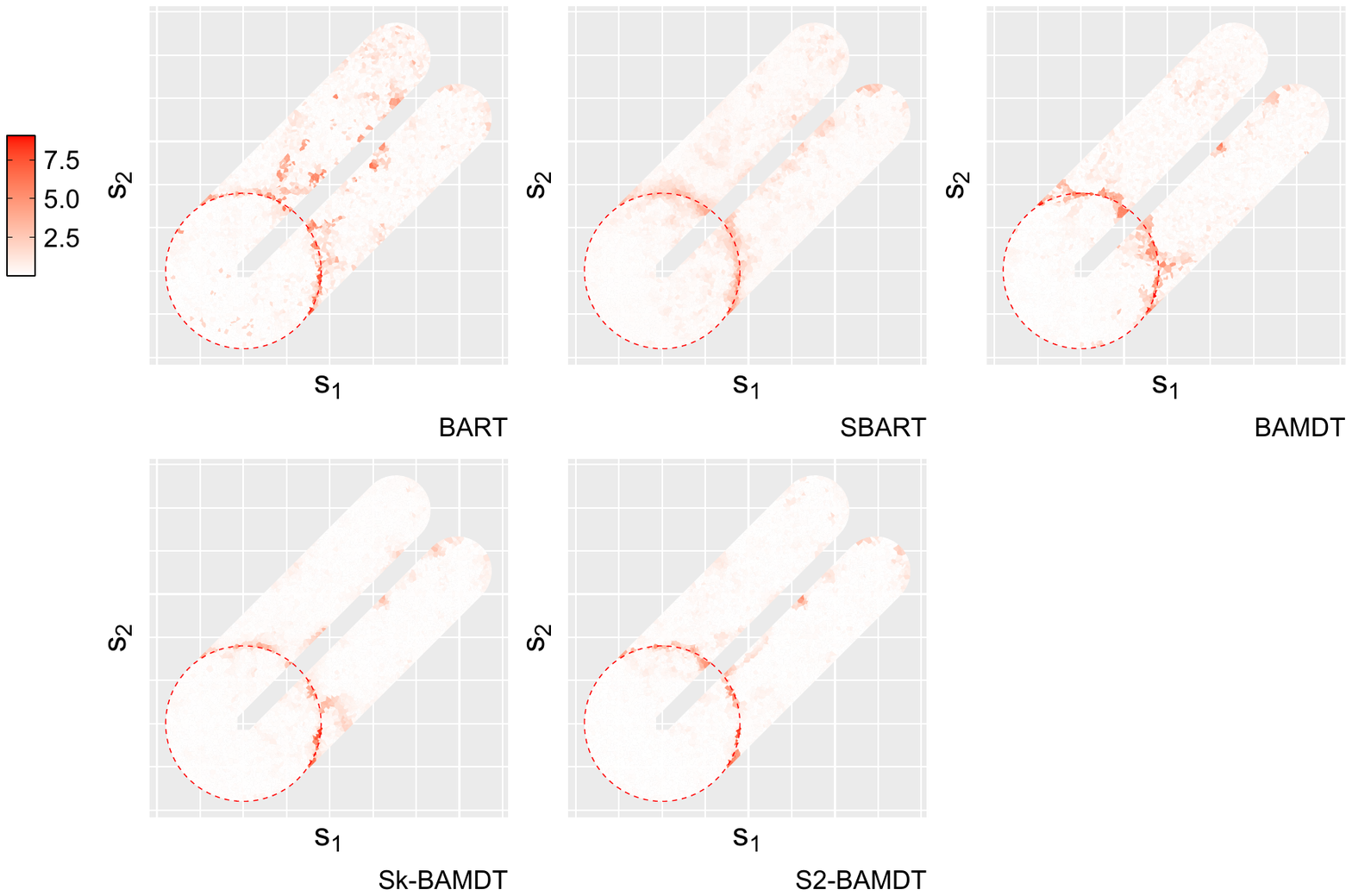}
  \caption{ The CRPS of SBAMDT, BAMDT, BART and SBART associated with U-shape Example.} \label{Fig::crps_ScU}
\end{figure}

Table~\ref{tab:PM_ScU} summarizes the average performance metrics of the benchmark models and SBAMDT over 50 replicates, demonstrating that Sk-BAMDT performs the best among these methods, followed by S2-BADMT and BAMDT. The SBAMDT and BAMDT models outperform BART and SBART due to the use of multivariate splits that can more flexibly handle circular-shaped function discontinuities and U-shape complex domains. 
The win of SBAMDT over BAMDT illustrates the need to consider soft splits of the feature space by employing probabilistic decisions at decision tree internal nodes, which enables SBAMDT to better approximate functions with adaptive smoothness inside each cluster. 

Figures~\ref{Fig::ScU},~\ref{Fig::MAPE_ScU} and~\ref{Fig::crps_ScU} display the mean predictive surfaces, absolute prediction error (APE) and CRPS from SBAMDT and benchmark models for one randomly selected experiment, respectively. These figures again confirm that, in general, SBAMDT provides the best fit to the ground truth at most locations with smaller errors both around the circular jump and in the interior of the clusters. In contrast, BART and SBART suffer from some axis-parallel artifacts especially near the circular jump boundary. We also notice that BART and SBART have large errors at many locations near the domain boundaries separating the lower and upper arms of U-shape, mainly due to the fact that axis-parallel split could group two regions separated by physical barriers into one cluster. Moreover, it is evident from these figures that SBAMDT successfully reduces the prediction errors inside the clusters compared with BAMDT.

The additive decision tree models provide feature's importance by taking the posterior mean of how many times the feature occurs in the ensemble of trees, which we also include as a metric for model comparison. Figure~\ref{fig:VarC_ScU} shows the significance metric for both structured and unstructured features for all models using one of the fifty simulated data sets. The analysis demonstrates that all methods are capable  of identifying important features, $x_1$ and $\bfs$. 
SBART performs better in identifying noisy features since it applies a Dirichlet prior to the splitting probability proportion of features.

\subsection{2-D Square Example}
We now examine the performance of SBAMDT in a scenario on a simple square domain with horizontal and vertical jumps, which aligns with the assumptions of BART and SBART and hence favors these two competing methods. Specifically, data is generated using a piecewise smooth function on $[0,1]^2$ given by 
 \begin{equation*}
 f(\bfd) =
  \begin{cases}
  \sin(7x_1)\cos(4x_2)\ & \text{for}\ x_1\geq 0,\ x_2\leq 0   \\
  1 + \frac{2}{7}(2x_1 +1)^2 + (2x_2+1)^2 & \text{for}\ x_1\geq 0,\ x_2>0 \\
  5 & \text{for}\ x_1<0,\ x_2\leq 0.2  \\
  -5& \text{for}\ x_1<0,\ x_2>0.2.\\
  \end{cases}
  \end{equation*} 
  

For simplicity, we consider a noise-free environment without adding irrelevant features. In the SBAMDT and BAMDT models, $(x_1,x_2)$ is used as multivariate structural features, and each individual $x$ is also included as a univariate unstructural feature. This approach enables the algorithm to detect function variations within each cluster where both features are necessary, while allowing each univariate feature to be used for identifying axis-parallel splits.

We use a training data set of size $n=500$, a test data set of size $n_{test}=200$ and repeat the experiments 50 times with a noise level $\sigma = 0.1$.
A random subset of 140 training data is used as knots for SBAMDT and BAMDT.

The average performance metrics of Sk-BAMDT and the benchmark models over 50 replicates are summarised in Table~\ref{tab:PM_Toy}. For one of the simulated data sets, Figures~\ref{Fig::Toy} and~\ref{Fig::MAPE_Toy} display the absolute prediction error (APE) and mean predictive surfaces obtained from SBAMDT and benchmark models. While BART is favoured in this example, Sk-BAMDT outperforms its competitors. Sk-BAMDT performs better than BART because it can adjust to the function's greater smoothness levels while taking into account a comparatively small number of trees. Sk-BAMDT's win against SBART highlights our primary distinction: instead of enforcing a soft decision rule on every node, Sk-BAMDT permits some internal nodes to make hard decisions. Soft splits are required to adjust to greater smoothness levels of the regression function, as demonstrated by SBAMDT's win against BAMDT.
\begin{table}[H]
\caption{The average performance metrics for SBAMDT and benchmark models over 50 replicates. Numbers in parenthesis correspond to the standard deviation of the forecasts among the different replicates associated with Toy Example.} \label{tab:PM_Toy}
\begin{tabular}{p{2.5 cm} p{1.8 cm} p{1.8 cm} p{1.8cm}  p{2.5 cm}  p{2.5 cm} }
 \hline
 \multicolumn{6}{c}{ SBAMDT and Benchmark models}  \\
 \hline
   & BAMDT  &BART  & SBART & Sk-BAMDT (q=8) & S2-BAMDT (q=10)  \\ \hline
MAPE$\times10^{-1}$ & 3.46(0.46)  & 3.39(0.52) & 5.87(0.41)  & \textbf{2.79(0.67)} & 3.06 (0.71)   \\ 
RMSPE$\times10^{-1}$& 10.34(1.06)& 10.76(1.39) & 10.45(0.59) & \textbf{9.54(3.76)} & 11.35 (3.46)  \\
CRPS$\times10^{-1}$ & 2.91(0.43)& 2.79(0.48)& 4.31(0.32) & \textbf{2.14(0.65)} & 2.52 (0.68)  \\
 \hline
\end{tabular}
\end{table}

\begin{figure}[h]   
\centering\includegraphics[width=15cm]{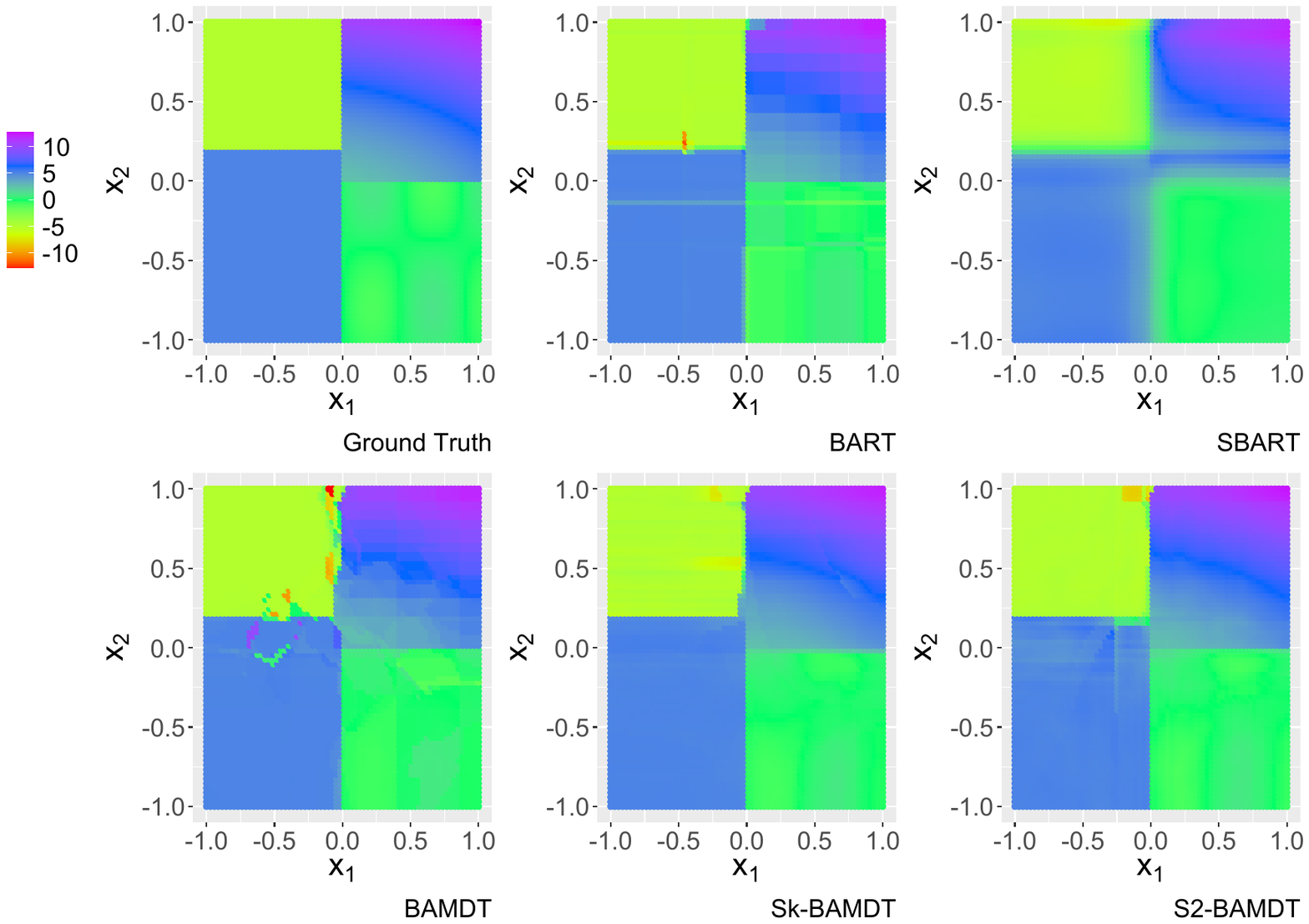}
  \caption{ The ground truth for $f(\bfs,\mathbf{x})$ and the predictive surfaces $\hat{f}(\bfs,\mathbf{x})$ of SBAMDT, BAMDT, BART and SBART associated with Toy Example.} \label{Fig::Toy}
\end{figure}

\begin{figure}[h]   
  \centering\includegraphics[width=15cm]{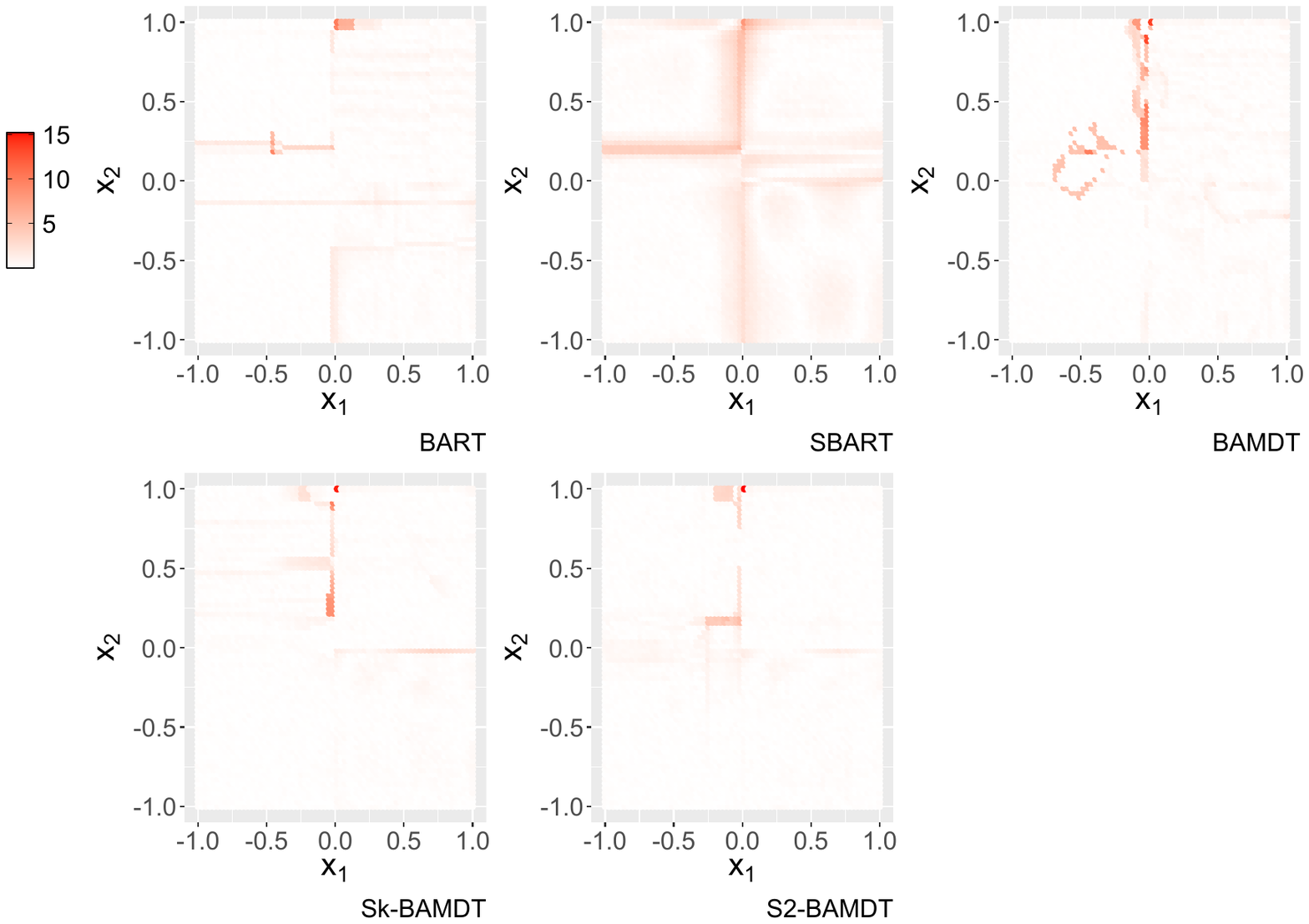}
  \caption{ The APE of SBAMDT, BAMDT, BART and SBART associated with Toy Example.} \label{Fig::MAPE_Toy}
\end{figure}

\section{Application to NYC Education}
\label{sec:NYD}
In this section, we use online data from NYC Education (2000) ~\citep{GeoDa} and apply SBAMDT and benchmark models (BAMDT, BART, SBART) to evaluate the predictive performance of SBAMDT. We model the logarithm of mean income as a function of the following variables: the total population under the age of eighteen ($x_1$); the percentage of all students enrolled in private schools ($x_2$); the percentage of the population aged sixteen to nineteen that has dropped out of high school ($x_3$); the percentage of the population aged twenty-five and over that has dropped out of high school ($x_4$); the percentage of the population aged twenty-five and over that has completed at least a bachelor's degree ($x_5$); the number of schools ($x_6$) and the location of the population (in latitude and longitude). We treat the location as a structured feature ($s_1$: longitude, $s_2$: latitude) and all other covariates as unstructured features ($\mathbf{x}$). 

The dataset comprises 1,690 entries, from which we selected 80\% for training and 20\% for testing. For the knots, a random selection of 420 training data points was used. We employed 100 grid points as potential univariate split cutoff values for each unstructured feature and executed MCMC algorithms for 20,000 iterations, retaining every tenth sample after discarding the first 10,000 iterations as burn-in. For SBART, we similarly discarded the first 10,000 iterations and saved 1,000 samples after thinning.

Performance metrics in log scale, presented in Table~\ref{tab:PM_NY}, indicate that SBAMDT achieves superior predictive accuracy compared to its competitors. Figure~\ref{Fig::NYTr} illustrates the fitted and observed income values at testing locations, demonstrating the model’s alignment with actual outcomes. Additionally, Figure~\ref{Fig::NY} showcases the mean prediction surfaces of both SBAMDT and benchmark models across various locations, with unstructured features set at their median values. Key statistics include a population of 650.08 individuals under eighteen, 24\% of students attending private schools, a high school dropout rate of 6.6\% for those aged 16–19, a 26\% dropout rate for individuals over 25, and 21\% of this group holding at least a bachelor's degree. Notably, the number of schools in the area is zero.

The comparison in Figure~\ref{Fig::NY} reveals that while BAMDT and SBAMDT conform to boundary requirements, BART and SBART fail to capture spatial patterns, particularly in Manhattan. Some areas of the SBART predictive surface are overly smooth, obscuring spatial characteristics. BART divides the domain into hyperrectangular segments, identifying axis-parallel boundaries without flexibility. In contrast, SBAMDT and BAMDT recognize more adaptable borders through multivariate splits. SBAMDT achieves smoother patterns and improved predictions, particularly in the Belt Parkway area, due to soft splits and probabilistic decisions at internal nodes.

Figure~\ref{fig:VarC_NY} presents the importance metrics for both structured and unstructured features across the models. Location emerges as the most significant determinant of income, followed by the number of schools ($x_6$) and the high school dropout rate for 16–19-year-olds ($x_3$). The differences in importance metrics between SBAMDT and BAMDT may stem from the softness control parameter, warranting further analysis in future work.

The population over 25 with at least a bachelor's degree is identified as the second most significant factor by Sk-BAMDT, BAMDT, and BART. We analyze its impact in areas such as Woodside, Long Island City, Astoria, Corona, and John F. Kennedy International Airport by fixing other unstructured variables at their median values. Figure~\ref{fig:VarCX_NY} depicts the marginal influence of higher educational attainment on income, showing a positive nonlinear relationship across regions, as anticipated. This effect varies by location, with individuals in Corona and Woodside earning more with higher education levels.

\begin{table}[H]
\caption{The performance metrics for SBAMDT and benchmark models (NYC Education Data).} \label{tab:PM_NY}
\begin{tabular}{p{2.5 cm} p{1.8 cm} p{1.8 cm} p{1.8 cm}  p{2.5 cm} p{2.5 cm} }
 \hline
 \multicolumn{6}{c}{Benchmark models}  \\
 \hline
   & BAMDT  &BART  & SBART & Sk-BAMDT (q=10) & S2-BAMDT (q=6) \\ \hline
MAPE$\times10^{-1}$& 1.77 & 1.83 & 1.76  & 1.66 & \textbf{1.65}  \\ 
RMSPE$\times10^{-1}$& 3.05 & 3.25 & 3.09 & 2.96 & \textbf{2.85}  \\
CRPS$\times10^{-1}$& 1.49 & 1.54& 1.50 & 1.34 & \textbf{1.32} \\
 \hline
\end{tabular}
\end{table}

\begin{figure}[h]   
    \centering\includegraphics[width=15cm]{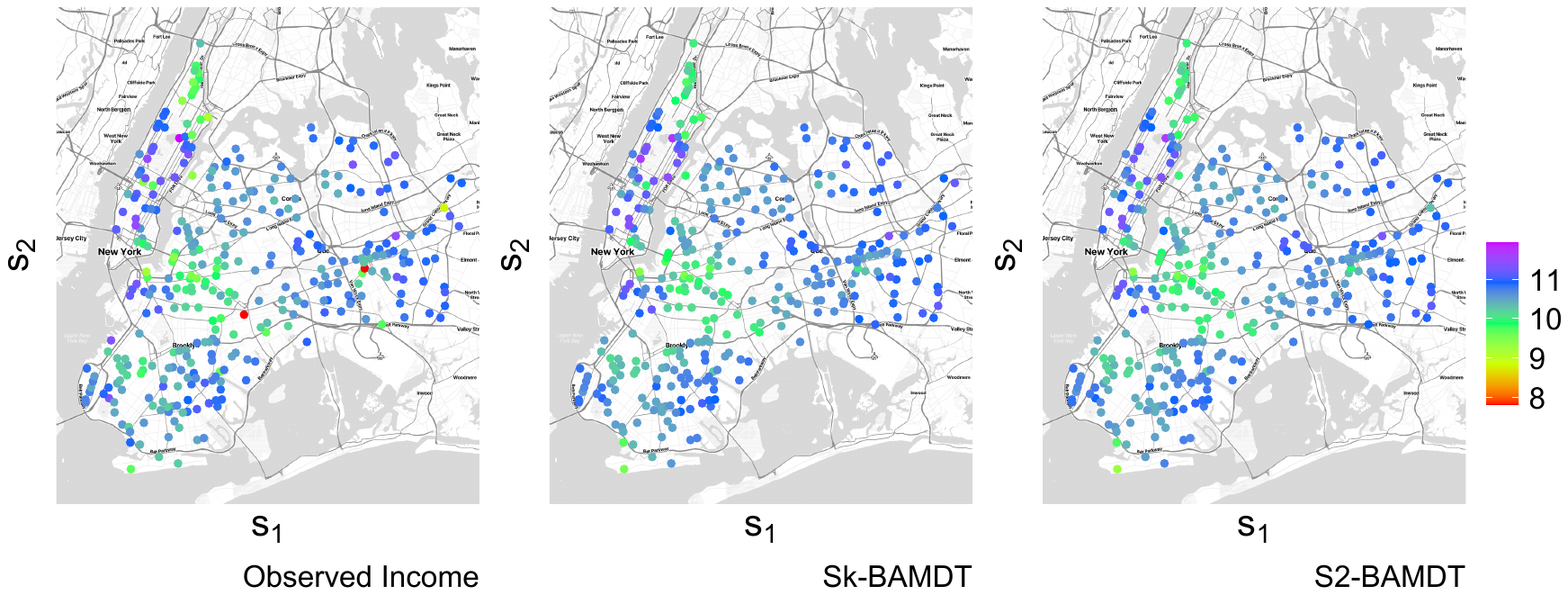}
  \caption{ The observed and fitted value of income at observed locations for SBAMDT.  } \label{Fig::NYTr}
\end{figure}

\begin{figure}[h]   
    \centering\includegraphics[width=15cm]{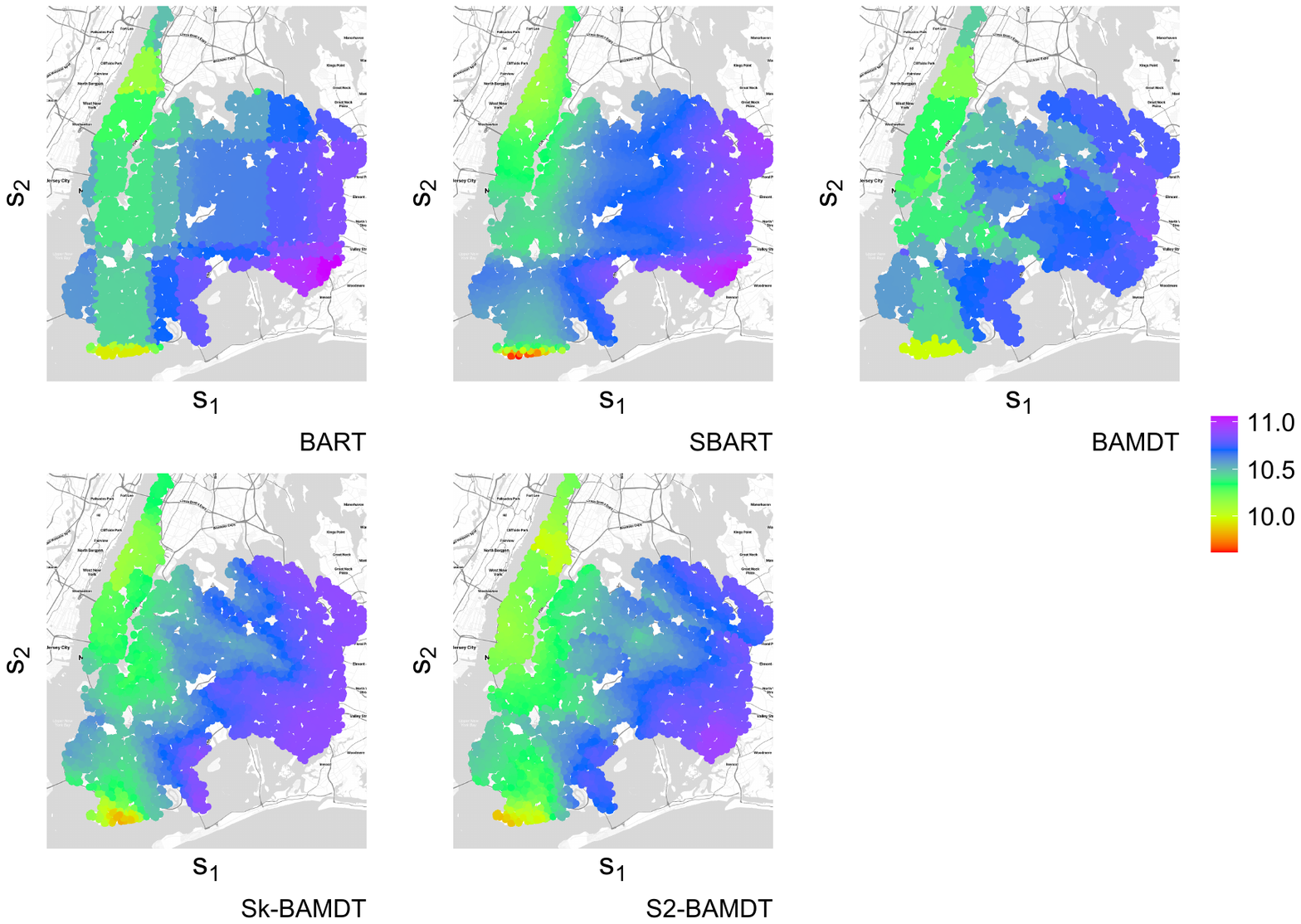}
  \caption{ The spatial partial dependence plots for SBAMDT, BAMDT, BART, and SBART. } \label{Fig::NY}
\end{figure}

\section{Conclusion}
\label{sec:conc}
In this paper, we introduce SBAMDT, a novel Bayesian additive decision tree model designed to model nonparametric functions with both structured and unstructured features. By employing an adaptive decision rule that allows both hard and soft decisions at each internal node, SBAMDT facilitates rapid transitions and adaptive smooth changes, enhancing its flexibility in capturing complex data relationships. This model allows for node-specific decision-making and accommodates structured features through multivariate splits as well as unstructured features via univariate splits. 
Our extensive simulation studies using both synthetic and real datasets show that SBAMDT consistently outperforms BART, SBART, and BAMDT, highlighting its efficacy in diverse applications.

Future work could focus on several areas for improvement. First, optimizing the code for faster execution would enhance its usability for large datasets. Second, incorporating a Dirichlet prior for feature splitting probabilities may refine the model's decision-making process further and improve its ability to handle higher dimensional features. Third, investigating the theoretical performance of function approximation through Bayesian posterior concentration theories could yield valuable insights into the model's behavior. Lastly, expanding the application of SBAMDT beyond nonparametric regression to tasks such as classification and causal inference presents exciting opportunities for further research.

\bigskip

\appendix
\section{The conditional integrated likelihood}
It holds that
\begin{enumerate}
        \item
        \begin{align*}
        P\left(\bfR^{(h)}|\bfM^{(h)},T^{(h)},\bfA^{(h)},\sigma^2,\sigma_{\mu}^2 \right) &= \left(2\pi\sigma^2\right)^{-n/2} \times  \\ &  \prod_{i=1}^n \exp\left(-\frac{1}{2\sigma^2}\left(R_i^2-2R_i\sum_{l=1}^{L_h}\phi_{hl}(\bfd_i)\mu_{hl} + \left(\sum_{l=1}^{L_h}\phi_{hl}(\bfd_i)\mu_{hl}\right)^2\right)\right)
        \end{align*}
          
         \item \begin{equation*}
         P(\bfM^{(h)}|T^{(h)},\sigma_{\mu}^2) = |2\pi\sigma_{\mu}^2|^{-1/2}\exp\left(-\frac{1}{2}\bfM^{(h)T}\frac{1}{\sigma_{\mu}^2}I\bfM^{(h)}\right)
         \end{equation*}
         \item
         \begin{align*}
         \left(\bfM^{(h)}-\Omega \frac{1}{\sigma^2}\sum_{i=1}^n\bm{\Phi}_i R_i^{(h)}\right)^{T}\Omega^{-1}\left(\bfM^{(h)}-\Omega \frac{1}{\sigma^2}\sum_{i=1}^n\bm{\Phi}_i R_i^{(h)}\right) &=
\\ \bfM^{(h)T}\Omega^{-1}\bfM^{(h)} - \frac{2}{\sigma^2} \sum_{i=1}^n R_i^{(h)}\bm{\Phi}_i^{T}\bfM^{(h)} + \hat{\mu}^T\Omega^{-1}\hat{\mu} \end{align*}
\end{enumerate}
Thus, the conditional integrated likelihood is given by
\begin{align*}
    P(\bfR^{(h)}|T^{(h)},\bfA^{(h)},\sigma^2,\sigma_{\mu}^2) &= (2\pi\sigma^2)^{-n/2} |2\pi \sigma_{\mu}^2I|^{-1/2}\exp\left(-\frac{||\bfR^{(h)}||^2}{2\sigma^2}\right)\exp\left(\frac{1}{2}\hat{\mu}^T\Omega^{-1}\hat{\mu}\right) \times\\
&\int {\exp\left((\bfM^{(h)}-\hat{\mu})^T\Omega^{-1}(\bfM^{(h)}-\hat{\mu})\right)d\bfM^{(h)}} \\
&=\frac{|2\pi\Omega|^{1/2}}{(2\pi\sigma^2)^{n/2}|2\pi\sigma^2I|^{1/2}}\exp\left(-\frac{||\bfR^{(h)}||^2}{2\sigma^2}+\frac{1}{2}\hat{\mu}^T\Omega^{-1}\hat{\mu}\right).
    \end{align*}

\section{The conditional distribution of $\bfM^{(h)}$} We have that \begin{align*}
P\left(\bfM^{(h)}|T^{(h)},\bfA^{(h)},\bfR^{(h)},\sigma^2,\sigma_{\mu}^2\right) &\propto P\left(\bfR^{(h)}|\bfM^{(h)},T^{(h)},\bfA^{(h)},\bfR^{(h)},\sigma^2,\sigma_{\mu}^2\right)P\left(\bfM^{(h)}|T^{(h)},\sigma^2_{\mu}\right) \\ 
&\propto \exp\left((\bfM^{(h)}-\hat{\pmb{\mu}}^{(h)})^{T}(\Omega^{(h)})^{-1}(\bfM^{(h)}-\hat{\pmb\mu}^{(h)})\right),
\end{align*} inducing that the conditional distribution of $\bfM^{(h)}$ is a multivariate Gaussian distribution with mean $\hat{\pmb{\mu}}^{(h)}$ and covariance matrix $\Omega^{(h)}$:
    \begin{equation*}
        \bfM^{(h)}|\bfR^{(h)}, T^{(h)},\bfA^{(h)},\sigma^2,\sigma_{\mu}^2 \sim \mbox{N}\left(\hat{\pmb{\mu}}^{(h)},\Omega^{(h)}\right).
    \end{equation*}
\section{The conditional distribution of $\sigma^2$} We have that
\begin{align*}
P\left(\sigma^2|\bfT,\pmb{M},Y,\bfA\right) &\propto P\left( Y|\bfT,\bfM,\sigma^2,\bfA \right)P\left(\sigma^2\right) \\
& \propto (\sigma^2)^{-1 -\frac{n+v}{2}} \exp\left(-\frac{1}{\sigma^2}\left(\frac{1}{2}\sum_{i=1}^n\left(Y_i-\sum_{h=1}^m\sum_{l=1}^{L_h}\mu_{hl}\phi_{hl}(\bfd_i)\right)^2+\frac{v\lambda}{2}\right) \right),
\end{align*} inducing that the conditional distribution $\sigma^2$ is an Inverse-gamma distribution with shape $\frac{n+v}{2}$ and scale $\frac{1}{2}\sum_{i=1}^n \left( Y_i -\sum_{h=1}^m\sum_{l=1}^{L_h}\mu_{hl}\phi_{hl}(\bfd_i) \right)^2 + \frac{v\lambda}{2}$.
\section{ The conditional distribution of $\sigma_{\mu}^2$} We have that 
    \begin{align*}
    P(\sigma_{\mu}^2|\bfT,\bfM,Y) & \propto P(\bfM|\bfT,\sigma_{\mu}^2)P(\sigma^2_{\mu}) \\ 
    & \propto (\sigma^2)^{-\alpha_{\mu} - \frac{1}{2}\sum_{h=1}^m L_h-1}
 \exp\left(-\frac{1}{\sigma^2_{\mu}} \left( \frac{1}{2}\sum_{h=1}^m\sum_{l=1}^{L_h} \mu_{hl}^2 + \beta_{\mu}\right)\right),
     \end{align*} inducing that the conditional distribution of $\sigma_{\mu}^2$
    is an Inverse-gamma distribution with shape $\alpha_{\mu} + \frac{1}{2}\sum_{h=1}^m L_h$ and scale $\frac{1}{2}\sum_{h=1}^m\sum_{l=1}^{L_h} \mu_{hl}^2 + \beta_{\mu}$. 

\section{The conditional distribution of $\bfp_A$ given Sk-BAMDT}
    We have that
    \begin{equation*}
        P(\bfp_A|\bfT,\bfA) \propto \prod_{h=1}^m\prod_\eta P(A_{\eta}^{(h)}|\bfp_A) P(\bfp_A) 
        \propto \prod_{l=0}^{k} p_{l}^{\sum_{h}\sum_\eta\mathds{1}\left(A_{\eta}^{(h)}=l\right) + \psi_{l}}, 
    \end{equation*}inducing that the conditional distribution of $\bfp_A$ is a Dirichlet distribution : 
    \begin{align*}
        P(\bfp_A|\bfT,\bfA) &\sim \mbox{Dirichlet}(\tilde{\psi}) \\
        \pmb{\tilde{\psi}}&=(\tilde{\psi}_0,\cdots,\tilde{\psi}_{k}) \\
        \tilde{\psi}_l &= \sum_{h} \sum_\eta \mathds{1}\left(A_{\eta}^{(h)}=l\right) +\psi_l.
    \end{align*}

    \section{The conditional distribution of $p_A$ given S2-BAMDT}
    We have that
    \begin{align*}
        P(p_A|\bfT,\bfA) &\propto \prod_{h=1}^m\prod_\eta P(A_{\eta}^{(h)}|p_A) P(p_A) \\
        &= p_A^{\sum_{h}\sum_\eta A_{\eta}^{(h)}+ s_a-1} (1-p_A)^{\sum_{h}\sum_\eta\left(1-A_{\eta}^{(h)}\right)+ s_b-1}, 
    \end{align*}inducing that the conditional distribution of $p_A$ is a Beta distribution : 
    \begin{align*}
         p_A|\bfT,\bfA &\sim \mbox{Beta}(\sum_{h}\sum_\eta A_{\eta}^{(h)}+s_a,\sum_{h}\sum_\eta (1-A_{\eta}^{(h)})+s_b ). 
    \end{align*}
    
    \section{The conditional distribution of $A_{\eta}^{(h)}$ given Sk-BAMDT} We have that 
    \begin{align*}
        P(A_{\eta}^{(h)}|T^{(h)},\mathbf{\bfp_A},\bfR^{(h)},\sigma^2,\sigma_{\mu}^2) &\propto P(\bfR^{(h)}|T^{(h)},A_{\eta}^{(h)},\sigma^2,\sigma_{\mu}^2)P(A_{\eta}^{(h)}|\mathbf{\bfp_A}) \\
        &= \prod_{l=0}^{k}\left(P\left(\bfR^{(h)}| T^{(h)}, A_{\eta}^{(h)}=l, \sigma^2, \sigma_{\mu}^2\right)p_{l}\right)^{\mathds{1}\left(A_{\eta}^{(h)}=l\right)} \\
        &\propto \prod_{l=0}^{k}\frac{p_{l}P\left(\bfR^{(h)}| T^{(h)}, A_{\eta}^{(h)}=l,\sigma^2, \sigma_{\mu}^2\right)}{\sum_j p_{j}P\left(\bfR^{(h)}| T^{(h)}, A_{\eta}^{(h)}=j, \sigma^2, \sigma_{\mu}^2\right)},  
        \end{align*} inducing that the conditional distribution of $A_{\eta}^{(h)}$ is a Multinomial distribution: 
       \begin{align*}
        A_{\eta}^{(h)}|T^{(h)},\mathbf{\bfp_A},\bfR^{(h)},\sigma^2,\sigma_{\mu}^2 &\sim \mbox{Multinomial} (\mathbf{w_{Ah}}) \\
        \mathbf{w_{Ah}}&=(w_{h0},\cdots,w_{h{k}}) \\
        w_{hl} &= \frac{p_l P\left(\bfR^{(h)}|T^{(h)},A_{\eta}^{(h)}=l,\sigma^2,\sigma_{\mu}^2\right)}{\sum_{j=0}^kp_j P\left(\bfR^{(h)}|T^{(h)},A_{\eta}^{(h)}=j,\sigma^2,\sigma_{\mu}^2\right)}.
    \end{align*}

\section{Proposals of Metropolis-Hastings for sampling the parameters $\{\alpha^{(h)}\}_{h=1}^m$} \label{Ap::MHCp}
We describe the proposals of Metropolis-Hastings to sample $\alpha^{(h)}|T^{(h)},\bfM^{(h)},\bfA^{(h)},\sigma^2,\sigma_{\mu}^2,\bfR^{(h)}$ (Algorithm~\ref{alg_MHa}). 
The chosen transition kernel q is a Gamma distribution with shape $d$ and rate $d/\alpha^{(h)}$, 
\begin{equation*}
a_{h}^* \sim \mbox{Gamma}\left(d,d/\alpha^{(h)}\right),
\end{equation*}
where $\alpha^{(h)}$ is the value of the former iteration. In our simulation study, we have used $d=20$.
The Hastings ratio can be expressed as the product of three terms:
\begin{itemize}
\item \textit{Transition ratio} 
\begin{equation*}
\mathrm{TR}=\frac{q\left(\alpha^{(h)}|a_{h}^*\right)}{q\left(a_{h}^{*}|\alpha^{(h)}\right)}=\frac{\left(\frac{1}{a_{h}^{*}}\right)^d \left(a^{(h)}\right)^{d-1}\exp\left(-\frac{d}{a_{h}^*}\alpha^{(h)}\right)} {\left(\frac{1}{\alpha^{(h)}}\right)^d \left(a_{h}^{*}\right)^{d-1}\exp\left(-\frac{d}{\alpha^{(h)}}a_{h}^{*}\right)}
\end{equation*}

\item \textit{Prior ratio}
\begin{equation*}
    \mathrm{PR}=\frac{P\left(a_{h}^*\right)}{P\left(\alpha^{(h)}\right)}=\left(\frac{a_{h}^*}{\alpha^{(h)}}\right)^{\alpha_g-1} \exp \left(-\beta_g\left(a_{h}^*-\alpha^{(h)}\right)\right)
\end{equation*}

\item \textit{Likelihood ratio} 
\begin{equation*}
\mathrm{LR}=\frac{P\left(\bfR^{(h)}|\bfM^{(h)},T^{(h)},\bfA^{(h)},a_{h}^*,\sigma^2 \right)}{P\left(\bfR^{(h)}|\bfM^{(h)},T^{(h)},\bfA^{(h)},\alpha^{(h)},\sigma^2 \right)}
\end{equation*}
Equation~\ref{eq:Lik} is applied twice to generate LR: once for the ratio's denominator and once for its numerator.

\end{itemize}

\begin{algorithm}[H] 
\caption{Metropolis-Hastings Algorithm for sampling from the posterior  $\alpha^{(h)}|T^{(h)},\bfM^{(h)},\bfA^{(h)},\sigma^2,\sigma_{\mu}^2,\bfR^{(h)}$ } 
\label{alg_MHa}
\begin{algorithmic}
\STATE{Generate a candidate value $a_{h}^{*}$ with probability $q\left(a_{h}^{*}|\alpha^{(h)}\right)$.}
\STATE {Set $\alpha^{(h)}=a_{h}^{*}$ with probability 
\begin{equation*}
\min \left\{1,\frac{q\left(\alpha^{(h)}|a_{h}^{*}\right)}{q\left(a_{h}^{*}|\alpha^{(h)}\right)} 
\frac{P\left(a_{h}^*\right)}{P\left(\alpha^{(h)}\right)}\frac{P\left(\bfR^{(h)}|\bfM^{(h)},T^{(h)},\bfA^{(h)},a_{h}^*,\sigma^2 \right)}{P\left(\bfR^{(h)}|\bfM^{(h)},T^{(h)},\bfA^{(h)},\alpha^{(h)},\sigma^2 \right)} \right\}  
\end{equation*} 
}
\end{algorithmic}
\end{algorithm}

\section{Proposals of Metropolis-Hastings for sampling tree structures} \label{Ap::MH_Trees} We describe the proposals of Metropolis-Hastings to sample the tree structures.
\paragraph{GROW proposal} This proposal randomly picks a terminal node, splits the chosen terminal node into two new nodes and assigns a decision rule to it. Let $\eta$ be the randomly picked terminal node in tree $T^{(h)}$. The Hastings ratio can be expressed as the product of three terms:
\begin{itemize}
\item \textit{Transition ratio} 
\begin{equation*}
\mathrm{TR}=\frac{q\left(T^{(h)}|T^{(h)*}\right)}{q\left(T^{(h)*}|T^{(h)}\right)}=\frac{P(\mathrm{PRUNE})N_s}{P(\mathrm{GROW})N_m p_{rule}(\eta)},
\end{equation*}
where $p_{rule}(\eta)$ is the probability of a decision rule assigned to node $\eta$, $N_m$ the number of internal nodes with two terminal childres and $N_s$ the number of terminal nodes in $T^{(h)}$.

\item \textit{Tree Structure ratio}
\begin{equation*}
    \mathrm{TSR}=\frac{P\left(T^{(h)*}\right)}{P\left(T^{(h)}\right)}=\frac{\gamma (1+d_{\eta})^{-\delta} \left(1-\gamma(2+d_{\eta})^{-\delta}\right)^2 p_{rule}(\eta)}{1-\gamma(1+d_{\eta})^{-\delta}},
\end{equation*}
where $d_{\eta}$ is the depth of node $\eta$.

\item \textit{Likelihood ratio} 
\begin{equation*}
\mathrm{LR}=\frac{P\left(\bfR^{(h)}|T^{(h)*},\sigma^2,\sigma_{\mu}^2,\bfA^{(h)*}\right)}{P\left(\bfR^{(h)}|T^{(h)},\sigma^2,\sigma_{\mu}^2,\bfA^{(h)}\right)}
\end{equation*}
 We apply Equation~\ref{eq:IntLik}, considering the proposed tree, $T^{(h)*}$, and each of the potential interior node decisions, $A_{\eta}^{(h)}$. We also use Equation~\ref{eq:IntLik}, considering the tree of the current iteration, $T^{(h)}$, and the determined interior node decisions of ancestors of $\eta$, $\bfA^{(h)}$. 

\begin{equation*}
\mathrm{LR}= \begin{cases}
\sum_{l=0}^k\frac{p_lP\left(\bfR^{(h)}|T^{(h)*},\bfA^{(h)},A_{\eta}^{(h)}=l,\sigma^2,\sigma^2_{\mu}\right)}{P\left(\bfR^{(h)}|T^{(h)},\bfA^{(h)},\sigma^2,\sigma^2_{\mu}\right)},\ \text{Sk-BAMDT} \\ \\
\frac{p_A P\left(\bfR^{(h)}|T^{(h)*},\bfA^{(h)},A_{\eta}^{(h)}=1,\sigma^2,\sigma^2_{\mu}\right)}{P\left(\bfR^{(h)}|T^{(h)},\bfA^{(h)},\sigma^2,\sigma^2_{\mu}\right)} + \frac{(1-p_A) P\left(\bfR^{(h)}|T^{(h)*},\bfA^{(h)},A_{\eta}^{(h)}=0,\sigma^2,\sigma^2_{\mu}\right)}{P\left(\bfR^{(h)}|T^{(h)},\bfA^{(h)},\sigma^2,\sigma^2_{\mu}\right)},\ \text{S2-BAMDT}
\end{cases}\\
\end{equation*}
\end{itemize}

\paragraph{PRUNE proposal} This proposal randomly picks a parent of two terminal nodes and turns it into a terminal node by collapsing the nodes below it. Let $\eta$ be the parent of two terminal nodes. The Hastings ratio can be expressed as the product of three terms:

\begin{itemize}
\item \textit{Transition ratio} 
\begin{equation*}
\mathrm{TR}=\frac{q\left(T^{(h)}|T^{(h)*}\right)}{q\left(T^{(h)*}|T^{(h)}\right)}=\frac{P(\mathrm{GROW})N_m p_{rule}(\eta)}{P(\mathrm{PRUNE})(N_s-1)}
\end{equation*}

\item \textit{Tree structure ratio}
\begin{equation*}
    \mathrm{TSR}=\frac{P(T^{(h)*})}{P(T^{(h)})}=\frac{ 1-\gamma(1+d_{\eta})^{-\delta}}{\gamma (1+d_{\eta})^{-\delta} \left(1-\gamma(2+d_{\eta})^{-\delta}\right)^2p_{rule}(\eta)}
\end{equation*}
\item \textit {Likelihood ratio}
\begin{equation*}
\mathrm{LR}=\frac{P\left(\bfR^{(h)}|T^{(h)*},\sigma^2,\sigma_{\mu}^2,\bfA^{(h)*}\right)}{P\left(\bfR^{(h)}|T^{(h)},\sigma^2,\sigma_{\mu}^2,\bfA^{(h)}\right)}
\end{equation*}
Similar to the GROW proposal, LR is derived by applying Equation~\ref{eq:IntLik} twice, once for the numerator and once for the denominator of the ratio.
\end{itemize}

\paragraph{Change proposal} This proposal randomly picks an internal node and randomly reassigns to it a probabilistic decision. For simplicity we are restricted to picking an internal node having two terminal nodes as children. The Hastings ratio can be expressed as the product of three terms:

\begin{itemize}
    \item \textit{Transition ratio}
    \begin{equation*}
        \mathrm{TR}=\frac{q\left(A_{\eta}^{(h)}|A_{\eta}^{(h)*}\right)}{q\left(A_{\eta}^{(h)*}|A_{\eta}^{(h)}\right)}=1
    \end{equation*}

    \item \textit{Assignment ratio}
    \begin{equation*}
        \mathrm{TAR}=\frac{P\left(A_{\eta}^{(h)*}\right)}{\left(A_{\eta}^{(h)}\right)}= \begin{cases}
 \frac{\prod_{j=0}^k p_j ^{\mathds{1}\left(A_{\eta}^{(h)*}=j\right)}}{\prod_{l=0}^k p_l ^{\mathds{1}\left(A_{\eta}^{(h)}=l\right)}},\ \text{Sk-BAMDT} \\ \\
\frac{p_A^{A_{\eta}^{(h)*}} (1-p_A)^{1-A_{\eta}^{(h)*}}}{p_A^{A_{\eta}^{(h)}} (1-p_A)^{1-A_{\eta}^{(h)}}},\ \text{Method 2} \\
\end{cases}\\
\end{equation*}

\item \textit{Likelihood ratio}

\begin{equation*}
\mathrm{LR}=\frac{P\left(\bfR^{(h)}|T^{(h)},\sigma^2,\sigma_{\mu}^2,\bfA^{(h)(-\eta)},A_{\eta}^{(h)*}\right)}{P\left(\bfR^{(h)}|T^{(h)},\sigma^2,\sigma_{\mu}^2,\bfA^{(h)(-\eta)},A_{\eta}^{(h)}\right)},
\end{equation*}
where $\bfA^{(h)(-\eta)}=\{A_b^{(h)}\}_{b\in T^{(h)}, b\neq \eta}$.
Equation~\ref{eq:IntLik} is applied twice to generate LR: once for the ratio's denominator and once for its numerator.
\end{itemize}

\section{Inference Algorithm for S2-BAMDT} \label{sec::S2BAMDT}
The primary difference between Sk-BAMDT and S2-BAMDT is how they handle probabilistic decisions for internal nodes. 
To infer the parameters $\left(\bfT,\bfM,\bfA,\pmb{\alpha},\sigma^2,\sigma^2_{\mu},p_A\right)$, we suggest using a Metropolis-Hastings within block Gibbs sampler (Algorithm~\ref{alg:MH_Method2}), similarly to Sk-BAMDT. The sampler requires $m$ successive draws from 
$P\left(T^{(h)},\bfM^{(h)},\bfA^{(h)}, \alpha^{(h)}|\sigma^2,\sigma_{\mu}^2,p_A,\bfR^{(h)}\right)$, followed by a draw of $\sigma^2$ from $P(\sigma^2|\bfT,\bfM,\bfA,\bfY)$, a draw of $\sigma^2_{\mu}$ from $P(\sigma^2_{\mu}|\bfT,\bfM,\bfA,\bfY)$ and a draw of $p_A$ from $P(p_A|\bfT,\bfA)$. In Appendix~\ref{Ap::MHCp}, we demonstrate the implementation of the Metropolis-Hastings algorithm that we propose to sample $\alpha^{(h)}|T^{(h)},\bfM^{(h)},\bfA^{(h)},\sigma^2,\sigma_{\mu}^2,\bfR^{(h)}$. 

Then, we present the expressions for the posterior distributions of $A_{\eta}^{(h)}$ and $p_A$ according to the assumptions of S2-BAMDT.

\begin{theorem}
\begin{enumerate}
       \item The conditional distribution of $p_A$ is a Beta distribution:
     \begin{equation*}
         P(p_A|\bfT,\bfA) \sim \mbox{Beta}\left(\sum_{h}\sum_\eta A_{\eta}^{(h)}+s_a,\sum_{h}\sum_\eta (1-A_{\eta}^{(h)})+s_b \right). 
     \end{equation*}
    \item The conditional distribution of $A_{\eta}^{(h)}$ is a Bernoulli distribution : 
    \begin{align*}
        A_{\eta}^{(h)}&|T^{(h)},p_A,\bfR^{(h)},\sigma^2,\sigma_{\mu}^2 \sim \mbox{Bernoulli} (w_{Ah}) \\
        w_{Ah} &= \frac{p_A P\left(\bfR^{(h)}|T^{(h)},A_{\eta}^{(h)}=1,\sigma^2,\sigma_{\mu}^2\right)}{p_A P\left(\bfR^{(h)}|T^{(h)},A_{\eta}^{(h)}=1,\sigma^2,\sigma_{\mu}^2\right) + (1-p_A) P\left(\bfR^{(h)}|T^{(h)},A_{\eta}^{(h)}=0,\sigma^2,\sigma_{\mu}^2\right)}.
    \end{align*}

\end{enumerate}
 
\end{theorem}

\begin{algorithm}[H] 
\caption{Metropolis-Hastings within Gibbs sampler (S2-BAMDT)} 
\begin{algorithmic}
\label{alg:MH_Method2}
\FOR{$t=1,2,3,..$ }
\FOR{$h=1$ to $m$ }
\STATE{Sample $\alpha^{(h)}|T^{(h)},\bfM^{(h)},\bfA^{(h)},\sigma^{2},\sigma_{\mu}^{2},\bfR^{(h)}$ }
\STATE{Sample $T^{(h)}|\bfR^{(h)},\sigma^2,\sigma_{\mu}^2,p_A$ using a Metropolis-Hastings Algorithm.}
\STATE{If $T^{(h)}$ grows, sample the probabilistic decision for the new internal node $\eta$, 
 $A_{\eta}^{(h)}|T^{(h)},p_A,\bfR^{(h)},\sigma^2,\sigma_{\mu}^2$, from $\mbox{Bernoulli} (w_{Ah})$.}
\STATE{Sample $ \bfM^{(h)}|\bfR^{(h)}, T^{(h)},\bfA^{(h)},\sigma^2,\sigma_{\mu}^2$ from $\mbox{N}\left(\hat{\pmb\mu}^{(h)},\Omega^{(h)}\right)$. }
\ENDFOR
\STATE{Sample $\sigma^{2}|\bfT,\bfM,\bfA,\bfY$ from $\mbox{Inverse-Gamma}(s_{1\sigma},s_{2\sigma})$. }
\STATE{Sample $\sigma^{2}_{\mu}|\bfT,\bfM,\bfA,\bfY$ from $\mbox{Inverse-Gamma}(s_{1\mu},s_{2\mu})$. }
\STATE{Sample $p_A|\bfT,\bfA$ from $\mbox{Beta}\left(\sum_{h}\sum_\eta A_{\eta}^{(h)}+s_a,\sum_{h}\sum_\eta \left(1-A_{\eta}^{(h)}\right)+s_b \right)$.}
\ENDFOR
\end{algorithmic}
\end{algorithm}

\section{U-shape Example 2}
In that example, we also consider a 45 degree rotated U-shape domain using latitude and longitude as structured features.  We  generate ten unstructured features uniformly distributed in the interval (0,1), $x_i\sim \mbox{U}\left(0,1\right)$. 
The U-shape domain is divided into three pieces by a circle with a radius of 0.9 centered at the origin. Assuming three Gaussian processes, we built a piecewise function $f(\cdot)$ that is dependent on $\left(\bfs,x_1, x_2\right)$ and has separate covariance matrices derived by squared exponential covariance functions for each process. We assume a training data set of size $n = 800$ and a test data set of size $n_{test} = 300$. The responses in the training and test data are generated using Equation~\ref{eq:model} at a noise level of $\sigma = 0.1$; 50 replicates are simulated. A random subset of 160 training data is used as knots for SBAMDT and BAMDT.

We discard the first 5000 iterations as burn-in and applied thinning by retaining one sample every five iterations from the remaining 5000 iterations of the MCMC algorithms in SBAMDT, BAMDT and BART. For SBART, we discard the first 8,000 iterations as burn-in and save 1000 samples after thinning. 

\begin{table}[h]
\caption{The average performance metrics for SBAMDT and benchmark models over 50 replicates. Numbers in parenthesis correspond to the standard deviation of the forecasts among the different replicates associated with U-shape Example 2.} \label{tab:PM_XPSmooth}
\begin{tabular}{p{2.5 cm} p{1.8 cm} p{1.8 cm} p{1.8cm}  p{2.5 cm} p{2.5 cm}    }
 \hline
 \multicolumn{6}{c}{SBAMDT and Benchmark models}  \\
 \hline
   & BAMDT  &BART  & SBART & Sk-BAMDT (q=6) & S2-BAMDT (q=4)\\ \hline
MAPE$\times10^{-1}$& 1.82(0.10)  &2.06(0.10)  & 2.00(0.08)  & 1.49(0.10)& \textbf{1.46} (0.10) \\ 
RMSPE$\times10^-1$& 2.50(0.16)& 2.81(0.15)& 2.72 (0.10) & 2.13(0.15) & \textbf{2.06}(0.17)  \\
CRPS$\times10^-1$& 1.45(0.10)&  1.65(0.09)& 1.53(0.06) & \textbf{1.13}(0.07)& 1.14(0.08)  \\
 \hline
\end{tabular}
\end{table}

\begin{figure}[h]   
    \centering\includegraphics[width=15cm]{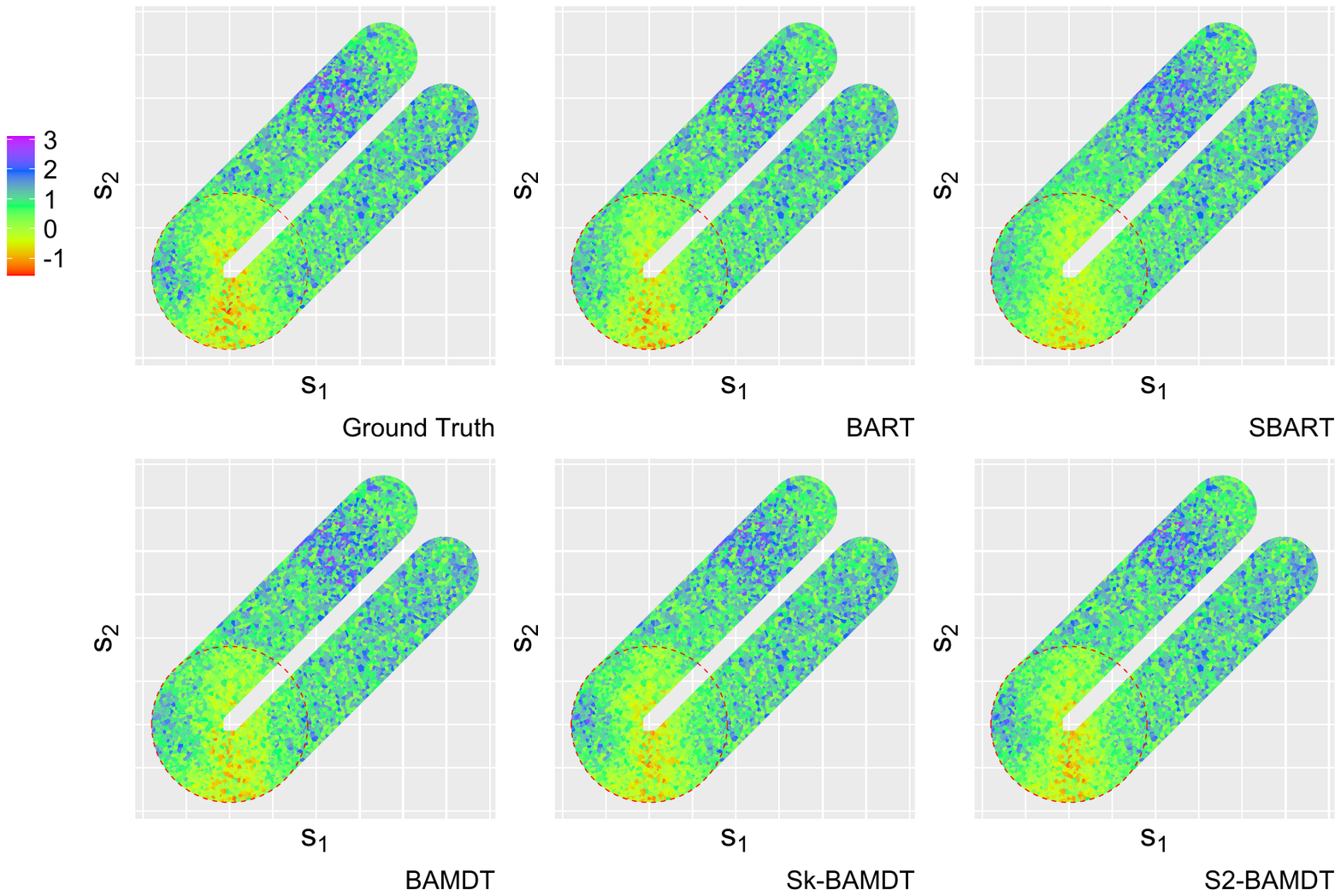}
      \caption{ The ground truth for $f(\bfs,\mathbf{x})$ and the spatial of SBAMDT, BAMDT, BART and SBART associated with U-shape Example 2. Red dashed circle indicate discontinuity surfaces in the true function projected to $\mathcal{M}$.} \label{Fig::XPSmooth}
\end{figure}

\begin{figure}[h]   
    \centering\includegraphics[width=15cm]{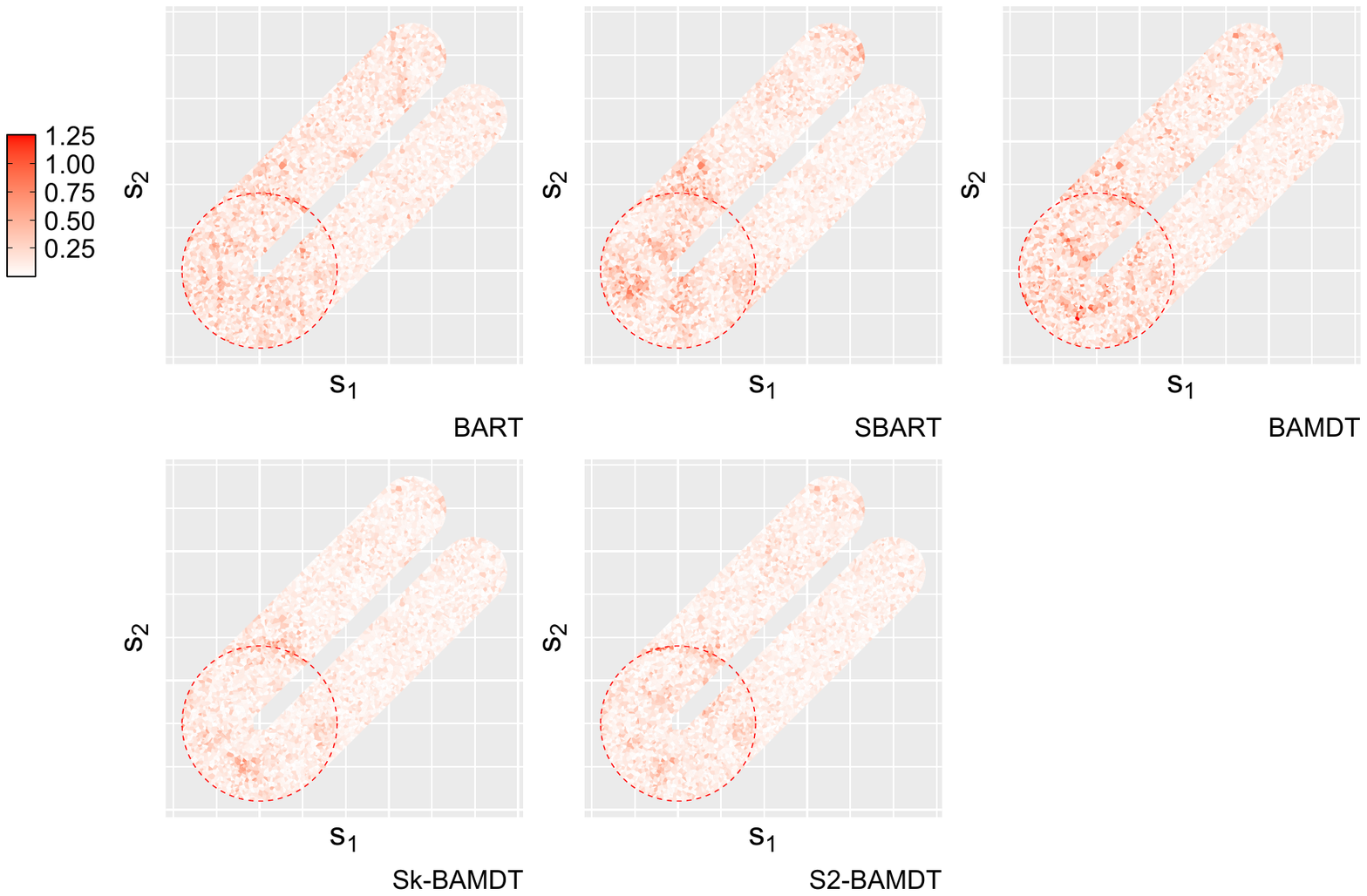}
     \caption{ The APE of SBAMDT, BAMDT, BART and SBART associated with U-shape Example 2.} \label{Fig::MAPE_XPSmooth}
\end{figure}

\begin{figure}[h]   
    \centering\includegraphics[width=15cm]{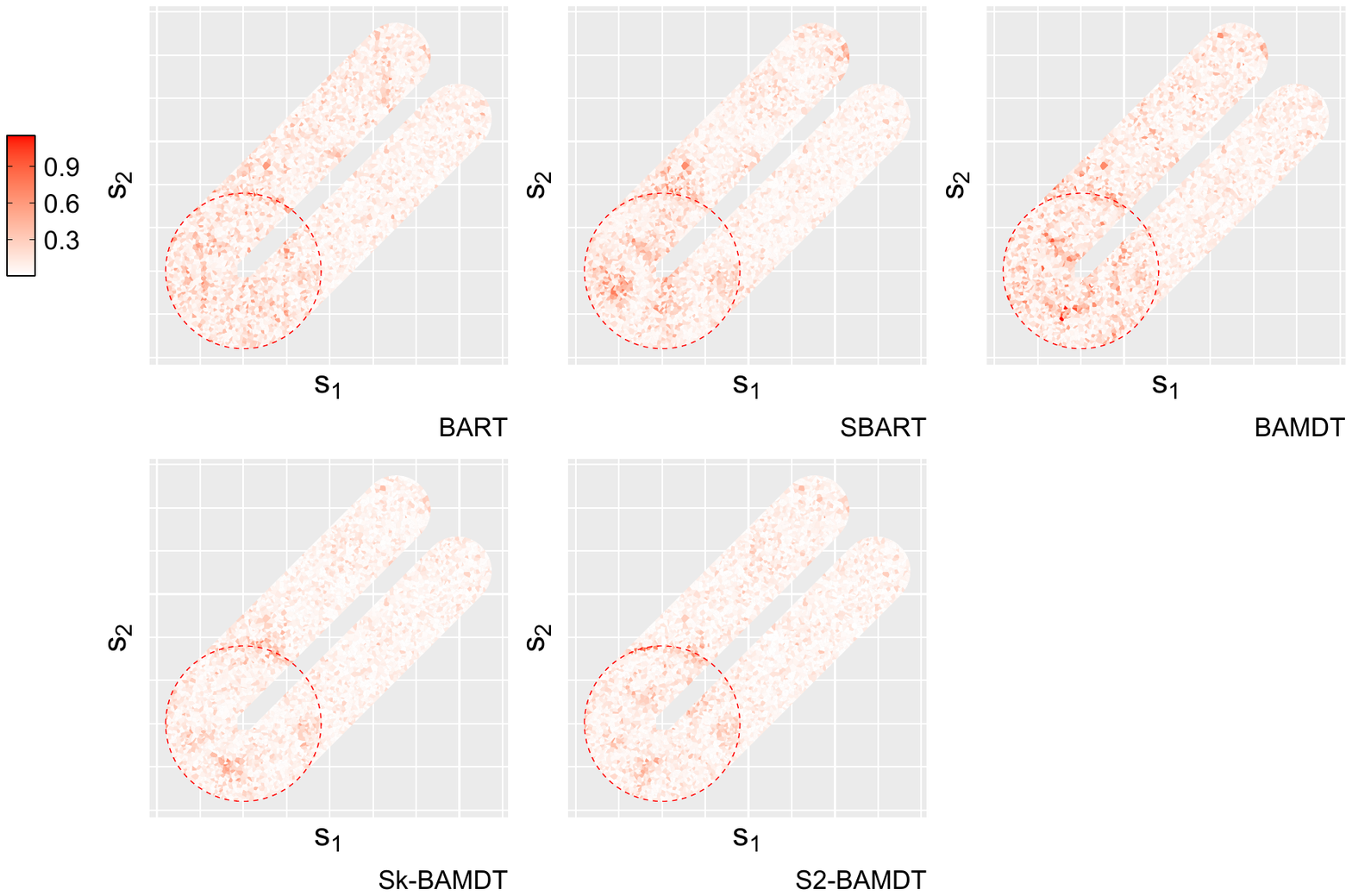}
    \caption{ The CRPS of SBAMDT, BAMDT, BART and SBART associated with U-shape Example 2.} \label{Fig::CRPS_XPSmooth}
\end{figure}

\begin{figure}[h]   
    \centering\includegraphics[width=15cm]{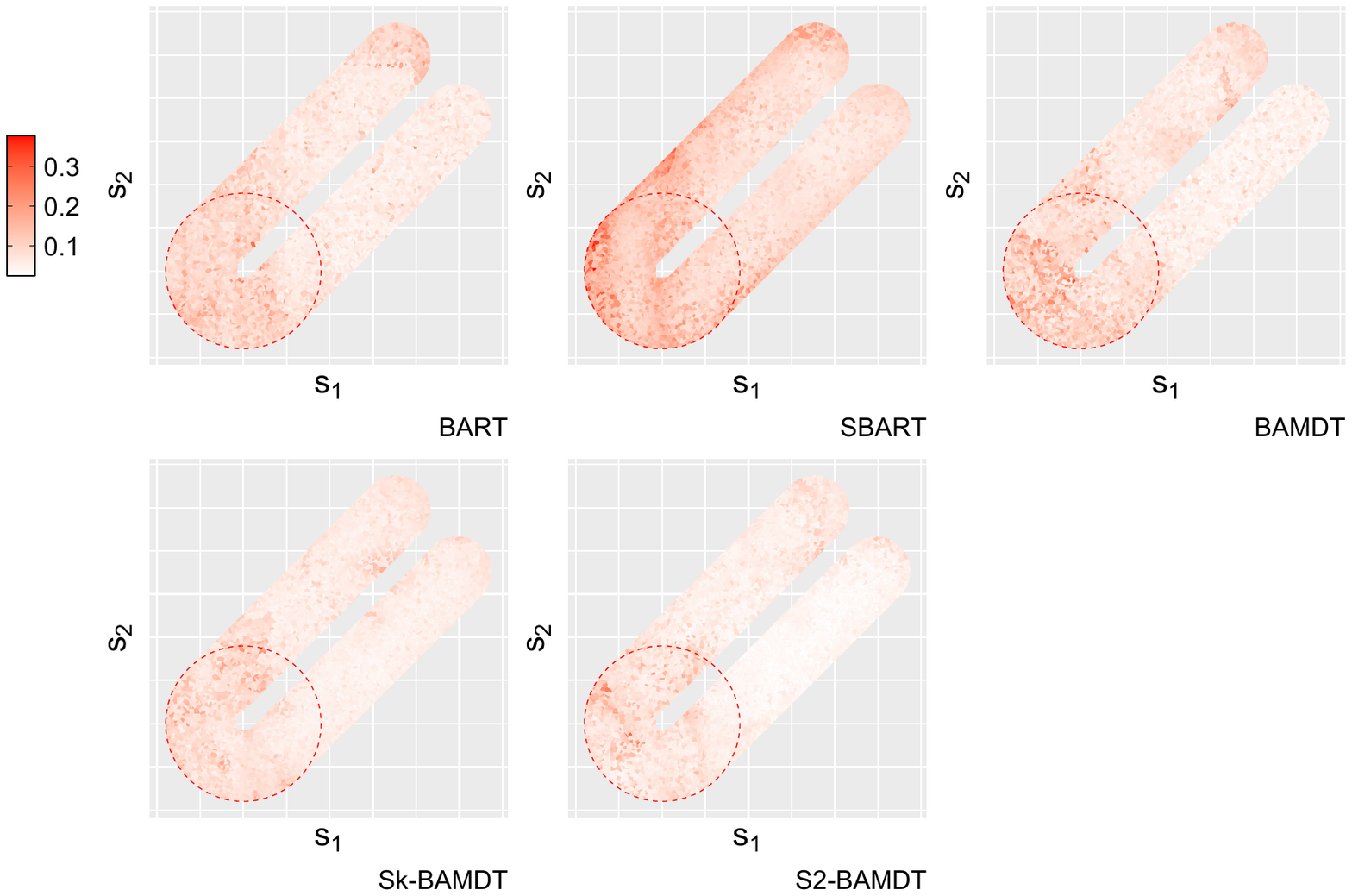}
  \caption{ The sd of SBAMDT, BAMDT, BART and SBART associated with U-shape Example 2.} \label{Fig::sd_XPSmooth}
\end{figure}

\begin{figure}
\begin{center}
\centering\includegraphics[width= 9cm]{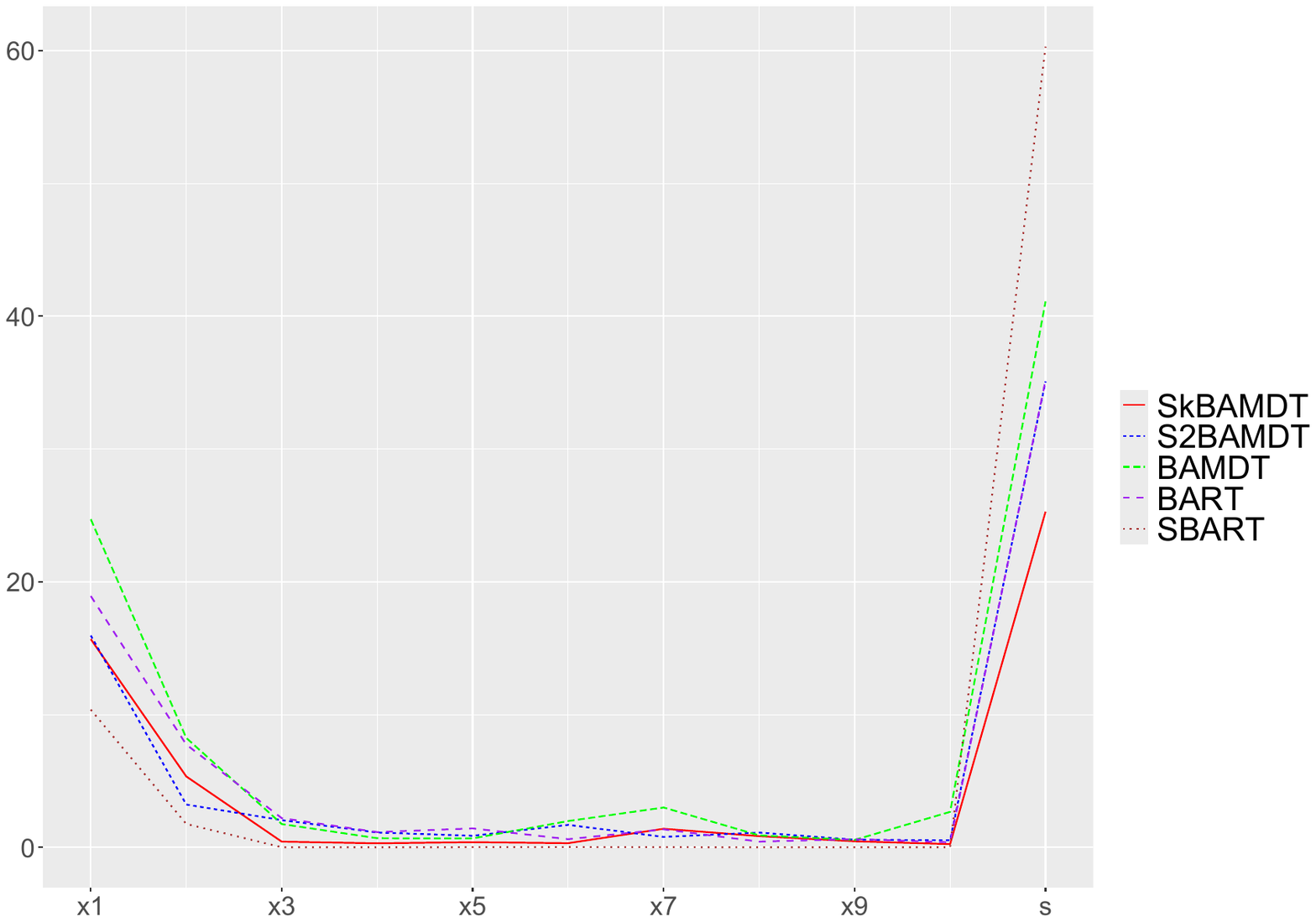}
\end{center}
\caption{The importance metric for each feature given by Sk-BAMDT, S2-BAMDT, BAMDT, BART and SBART associated with U-shape Example 2. \label{fig:VarC_XPSmooth}}
\end{figure}

The average performance measures of the benchmark models and SBAMDT over 50 replicates are summarised in Table~\ref{tab:PM_XPSmooth}, which shows that SBAMDT performs better than its competitors. The analysis highlights the need for making probabilistic decisions about hard or soft boundaries at each internal node as well as multivariate and soft splits in the feature space.

For a single randomly selected simulated test data set, the mean predictive surfaces, absolute percentage error (APE), CRPS and standard deviation (sd) of predictions from SBAMDT and benchmark models are shown in Figures~\ref{Fig::XPSmooth},~\ref{Fig::MAPE_XPSmooth},~\ref{Fig::CRPS_XPSmooth} and ~\ref{Fig::sd_XPSmooth}. The results show that SBAMDT produces predictions with less uncertainty and offers the best fit to the ground truth. Within clusters, SBAMDT operates very well with low errors close to the cluster boundary. We don't see that behaviour in its competitors, which supports SBAMDT's win. Using one of the fifty simulated data sets, Figure~\ref{fig:VarC_XPSmooth} displays the significance metric for both structured and unstructured characteristics for every model. Similar to our other simulation results, SBAMDT prioritize the non-noisy features.

\section{Proof of Theorem 1} \label{Theorem1}
\begin{enumerate}
    \item 
We consider that $\bfT$ and $\bfA$ are known. For $\mu_{hl}$, we have used a normal prior with a zero mean and variance $\sigma^2_{\mu}$. The variance $\sigma^2_{\mu}$ is subjected to an inverse-Gamma prior with hyperparameters $\alpha_{\mu}$ and $\beta_{\mu}$.\\
The expected value of $f$ by integrating the function
over the prior distribution of 
$\bfM$ is given by:
\begin{equation}
    E\left(f(\bfd)|\bfT,\bfA\right)=\sum_{h=1}^m\sum_{l=1}^{L_h} \phi_{hl}(\bfd)E\left(\mu_{hl}|T^{(h)}\right)=0.
\end{equation}
The covariance function of $f$ with respect to the prior distribution of $\bfM$ is given by:
\begin{align*}
    C(f_i,f_j)&=E\left((f(\bfd_i)-0)(f(\bfd_j)-0)|\bfT,\bfA\right)\\
    &=\sum_{h=1}^m\sum_{l=1}^{L_h}\phi_{hl}(\bfd_i)\phi_{hl}(\bfd_j) \mathrm{Var}\left(\mu_{hl}\right)\\
    &=\frac{\beta_{\mu}}{\alpha_{\mu}-1}\sum_{h=1}^m\sum_{l=1}^{L_h}\phi_{hl}(\bfd_i)\phi_{hl}(\bfd_j).
\end{align*}
The gradient of $f$ with respect to $\bfd$ is given as follows:
\begin{equation*}
    \nabla_\bfd f(\bfd)=\sum_{h=1}^m\sum_{l=1}^{L_h} \mu_{hl}\nabla_{\bfd}\phi_{hl}(\bfd).
\end{equation*}
We use the product rule to find the gradient of the probability of a data point $\bfd$ falling into the terminal node $l$:
\begin{align*}
    \nabla_{\bfd}\phi_{hl}(\bfd)&=\sum_{\eta \in P_{hl}(\bfd)} \nabla_{\bfd} P_{\eta}\left(\bfd;A_\eta^{(h)}\right)\prod_{v \in P_{hl}(\bfd), v\neq \eta}P_v\left(\bfd; A_v^{(h)} \right), \quad \text{where}\\  
    \nabla_{\bfd} P_{\eta}\left(\bfd; A_{\eta}^{(h)}\right)&=\left(\nabla_{\bfd}[z_{\eta L}^{(h)}(\bfd)| A_{\eta}^{(h)}]\right)
    ^{\mathds{1}\left(r_{\eta L}^{(h)}=1\right)}\big(-\nabla_{\bfd}[z_{\eta L}^{(h)}(\bfd)| A_{\eta}^{(h)} ]\big)^{1-\mathds{1}\left(r_{\eta L}^{(h)}=1\right)}.
\end{align*}

\begin{enumerate}
    \item \textbf{Univariate split:} Let $x_{i}$ be the uniformly chosen unstructured feature, and $x^*_{iL}$ and $x^*_{iR}$ the values of the $i_{th}$ unstructured feature for the nearest left and right knot at node $\eta$. For a univariate split, the derivative of the decision rule function $[z_{\eta L}^{(h)}(\bfd)| A_{\eta}^{(h)}]$ is given by:
    \begin{align*}
        \nabla_{\bfd}[z_{\eta L}^{(h)}(\bfd)| A_{\eta}^{(h)}\neq 0]= 
        \left[\begin{array}{c}
            0 \\ \vdots \\ 
            C_i \\ \vdots \\ 0
        \end{array}\right].
    \end{align*}
    \begin{enumerate}
        \item Soft split ($A_{\eta}^{(h)}\neq 0$):
        \begin{equation}
        \begin{aligned}[b]
        C_i &= \diff{ \left(1 + e^{-\alpha_c^{(h)}\frac{|x_i - x^*_{iR}|-|x_i - x^*_{iL}|}{C_\eta^{(h)}}}\right)^{-1}}{x_i} \\ 
        &=  \left(1 + e^{-\alpha_c^{(h)}\frac{|x_i - x^*_{iR}|-|x_i - x^*_{iL}|}{C_\eta^{(h)}}}\right)^{-2} e^{-\alpha_c^{(h)}\frac{|x_i - x^*_{iR}|-|x_i - x^*_{iL}|}{C_\eta^{(h)}}} \\ 
        &\quad \times \frac{2\alpha_c^{(h)}}{C_\eta^{(h)}}\big(-\mathds{1}\left( x ^{*}_{iL}< x_{i}<x^*_{iR}\right) + \mathds{1}\left( x ^{*}_{iR}< x_{i}<x^*_{iL}\right)\big).
        \end{aligned}
        \label{eq:UniSplit}
        \end{equation}
        \item Hard split ($A_{\eta}^{(h)}= 0$): 
        \begin{align*}
        C_i = -\delta\left(x_i - \frac{x^*_{iR}+x^*_{iL}}{2}\right)\mathds{1}\left(x^*_{iL}<x_i<x^*_{iR}\right) + \delta\left(x_i - \frac{x^*_{iR}+x^*_{iL}}{2}\right)\mathds{1}\left(x^*_{iR}<x_i<x^*_{iL}\right),
        \end{align*}
        where $\delta(\cdot)$ denotes the Dirac delta function. The partial derivative with respect to $x_i$ is 0 except for $x_i=\frac{x_{iR}^* + x_{iL}^*}{2}$, where the impulse makes the derivative undefined. SBAMDT can, however, approach the hard split via a soft split if the softness control parameter $\alpha_c^{(h)}$ is set to infinity. 
    \end{enumerate}

    \item \textbf{Multivariate split:} Let $\bfs^*_{L}$ and $\bfs^*_{R}$ be the structured feature of the nearest left and right knot at node $\eta$. For multivariate splits:
    \begin{enumerate}
        \item Soft split ($A_{\eta}^{(h)}\neq 0$):
        \begin{equation}
        \begin{aligned}[b]
        C_s &= \nabla{\bfs} \left(1 + e^{-\alpha_c^{(h)}\frac{\|\bfs - \bfs^*_{R}\|_2 - \|\bfs - \bfs^*_{L}\|_2}{C_\eta^{(h)}}}\right)^{-1} \\
         &= \left(1 + e^{-\alpha_c^{(h)}\frac{\|\bfs - \bfs^*_{R}\|_2 - \|\bfs - \bfs^*_{L}\|_2}{C_\eta^{(h)}}}\right)^{-2} e^{-\alpha_c^{(h)}\frac{\|\bfs - \bfs^*_{R}\|_2 - \|\bfs - \bfs^*_{L}\|_2}{C_\eta^{(h)}}} \\
        &\quad \times \left(\frac{\alpha_c^{(h)} }{C_\eta^{(h)}}\right) 
        \left(\frac{\bfs - \bfs^*_R}{\|\bfs - \bfs^*_R\|_2} - \frac{\bfs - \bfs^*_L}{\|\bfs - \bfs^*_L\|_2}\right).
        \end{aligned}
        \label{eq:MultiSplit}
        \end{equation}
        \item Hard split ($A_{\eta}^{(h)}= 0$): 
        \begin{equation}
        C_s = -\delta\left( (\bfs_R^* - \bfs_L^*)^T\bfs - \frac{\bfs_R^{*T}\bfs_R^* - \bfs_L^{*T} \bfs_L^* }{2} \right) (\bfs_R^* - \bfs_L^*).
        \end{equation}
        The derivative does not exist at:
        \[
        (\bfs_R^* - \bfs_L^*)^T\bfs = \frac{\bfs_R^{*T}\bfs_R^* - \bfs_L^{*T} \bfs_L^* }{2}.
        \]
    \end{enumerate}
\end{enumerate}

\item
We assume that $\mathbf{T}$ is known. The expected value of \( g^{(h)}(\mathbf{d}) = \sum_{l=1}^{L_h} \mu_{hl} \phi_{hl}(\mathbf{d}) \), when integrated over the prior distributions of \( \mathbf{M} \) and \( \mathbf{A} \), is given by:
\begin{equation*}
    E\left(g^{(h)}(\mathbf{d})\right) = \sum_{l=1}^{L_h} E\left(\mu_{hl} \phi_{hl}(\mathbf{d}) \mid \mathbf{T}^{(h)}\right) = 0.
\end{equation*}

The variance of \( g^{(h)}(\mathbf{d}) \) is:
\begin{equation*}
    \text{Var}\left(g^{(h)}(\mathbf{d})\right) = \sum_{l=1}^{L_h} \text{Var}(\mu_{hl}) E(\phi_{hl}^2(\mathbf{d})) < \infty.
\end{equation*}

By the Central Limit Theorem (CLT), as the number of trees \( m \) goes to infinity, the function \( f(\mathbf{d}) = \sum_{h=1}^{m} g^{(h)}(\mathbf{d}) \) approaches a normal distribution with:
\begin{itemize}
    \item \textbf{Mean:} \( \sum_{h=1}^m E\left(g^{(h)}(\mathbf{d})\right) \),
    \item \textbf{Variance:} \( \sum_{h=1}^m \text{Var}\left(g^{(h)}(\mathbf{d})\right) \).
\end{itemize}

Thus, as \( m \to \infty \), we can model the prior of \( f \) as a **Gaussian Process (GP)**, where the expected value of \( f(\mathbf{d}) \), integrating over the prior distributions of \( \mathbf{M} \) and \( \mathbf{A} \), is:
\begin{equation*}
    E\left(f(\mathbf{d}) \mid \mathbf{T}\right) = \sum_{h=1}^m \sum_{l=1}^{L_h} E\left(\mu_{hl} \phi_{hl}(\mathbf{d}) \mid \mathbf{T}^{(h)}\right) 
    = \sum_{h=1}^m \sum_{l=1}^{L_h} E\left(\mu_{hl} \mid \mathbf{T}^{(h)}\right) E\left(\phi_{hl}(\mathbf{d})\right) = 0.
\end{equation*}

The covariance function between \( f(\mathbf{d}_i) \) and \( f(\mathbf{d}_j) \) with respect to the prior distributions of \( \mathbf{M} \) and \( \mathbf{A} \) is:
\begin{equation*}
    C(f(\mathbf{d}_i), f(\mathbf{d}_j)) = E\left[(f(\mathbf{d}_i) - 0)(f(\mathbf{d}_j) - 0) \mid \mathbf{T}\right] = \sum_{h=1}^m \sum_{l=1}^{L_h} E\left(\phi_{hl}(\mathbf{d}_i) \phi_{hl}(\mathbf{d}_j)\right) \text{Var}(\mu_{hl}).
\end{equation*}

Since \( \mu_{hl} \) follows a prior distribution with variance \( \text{Var}(\mu_{hl}) = \frac{\beta_{\mu}}{\alpha_{\mu} - 1} \), we can express the covariance as:
\begin{equation*}
    C(f(\mathbf{d}_i), f(\mathbf{d}_j)) = \frac{\beta_{\mu}}{\alpha_{\mu} - 1} \sum_{h=1}^m \sum_{l=1}^{L_h} E\left(\phi_{hl}(\mathbf{d}_i) \phi_{hl}(\mathbf{d}_j)\right).
\end{equation*}

Next, we compute \( E\left(\phi_{hl}(\mathbf{d}_i) \phi_{hl}(\mathbf{d}_j)\right) \). This expectation can be written as:
\begin{equation*}
    E\left(\phi_{hl}(\mathbf{d}_i) \phi_{hl}(\mathbf{d}_j)\right) = E\left(\phi_{hl}(\mathbf{d}_i; \mathbf{A}^{(h)}) \phi_{hl}(\mathbf{d}_j; \mathbf{A}^{(h)})\right).
\end{equation*}

Expanding this expectation as a product over all indices \( \eta \in P_{hl}(\mathbf{d}_i) \), we get:
\begin{equation*}
    E\left(\phi_{hl}(\mathbf{d}_i) \phi_{hl}(\mathbf{d}_j)\right) = \prod_{\eta \in P_{hl}(\mathbf{d}_i)} \int dA^{(h)}_{\eta} \, P(A_{\eta}^{(h)}) \, P_{\eta}(\mathbf{d}_i; A_{\eta}^{(h)}) P_{\eta}(\mathbf{d}_j; A_{\eta}^{(h)}).
\end{equation*}

This simplifies to:
\begin{equation*}
    E\left(\phi_{hl}(\mathbf{d}_i) \phi_{hl}(\mathbf{d}_j)\right) = \prod_{\eta \in P_{hl}(\mathbf{d}_i)} \sum_{v=0}^k P_{\eta}(\mathbf{d}_i; A_{\eta}^{(h)} = v) \, P_{\eta}(\mathbf{d}_j; A_{\eta}^{(h)} = v) \, P(A_{\eta}^{(h)} = v).
\end{equation*}

Thus, the covariance between \( f(\mathbf{d}_i) \) and \( f(\mathbf{d}_j) \) can be compactly expressed as:
\begin{equation*}
    C(f(\mathbf{d}_i), f(\mathbf{d}_j)) = \frac{\beta_{\mu}}{\alpha_{\mu} - 1} \sum_{h=1}^m \sum_{l=1}^{L_h} \prod_{\eta \in P_{hl}(\mathbf{d}_i)} \sum_{v=0}^k P_{\eta}(\mathbf{d}_i; A_{\eta}^{(h)} = v) \, P_{\eta}(\mathbf{d}_j; A_{\eta}^{(h)} = v) \, P(A_{\eta}^{(h)} = v).
\end{equation*}
\item Conditional on \( \{\hat{\bfT}, \hat{\bfM}^{(-h)}, \hat{\bfA}, \hat{\sigma}^2, \hat{\sigma_{\mu}}^2\} \), we have that:
\begin{equation*}
    f(\bfd) = \bfM^{(h)T}\hat{\bm\Phi}^{(h)}(\bfd) + \sum_{h'=1, h'\neq h}^m \hat{\bfM}^{(h')T}\hat{\bm\Phi}^{(h')}(\bfd).
\end{equation*}

The conditional distribution of \( \bfM^{(h)} \) is a multivariate Gaussian distribution with a mean \( \hat{\bm\mu}^{(h)} \) and covariance matrix \( \Omega^{(h)} \). The expected value of \( f(\bfd) \), by integrating the function over the conditional distribution of \( \bfM^{(h)} \), is given by:
\begin{align*}
    \hat{\mu}_f^{(h)}(\bfd) &= E\left(f(\bfd) \mid \hat{T}, \hat{\bfM}^{(-h)}, \hat{\bfA}, \hat{\sigma}^2, \hat{\sigma}_{\mu}^2, \bfY \right) \\
    &= \hat{\mu}^{(h)T}\hat{\bm\Phi}^{(h)}(\bfd) + \sum_{h'=1, h' \neq h}^m \hat{\bfM}^{(h')T} \hat{\bm{\Phi}}^{(h')}(\bfd).
\end{align*}

This implies that:
\[
f(\bfd) - \hat{\mu}_f^{(h)}(\bfd) = \left(\bfM^{(h)} - \hat{\mu}^{(h)}\right)^T \hat{\bm\Phi}^{(h)}(\bfd),
\].

The covariance function of \( f \) with respect to the conditional distribution of \( \bfM^{(h)} \) is:
\begin{align*}
    \hat{C}(f_i, f_j) &= E\left(\left(f(\bfd_i) - \hat{\mu}_f^{(h)}(\bfd_i)\right)\left(f(\bfd_j) - \hat{\mu}_f^{(h)}(\bfd_j)\right) \mid \hat{T}, \hat{\bfM}^{(-h)}, \hat{\bfA}, \hat{\sigma}^2, \hat{\sigma}_{\mu}^2, \bfY \right) \\
    &= \hat{\bm\Phi}^{(h)T}(\bfd_i) E\left(\left(\bfM^{(h)} - \hat{\mu}^{(h)}\right)\left(\bfM^{(h)} - \hat{\mu}^{(h)}\right)^T \mid \hat{T}, \hat{\bfM}^{(-h)}, \hat{\bfA}, \hat{\sigma}^2, \hat{\sigma}_{\mu}^2, \bfY \right) \hat{\bm\Phi}^{(h)}(\bfd_j) \\
    &= \hat{\bm\Phi}^{(h)T}(\bfd_i) \Omega^{(h)} \hat{\bm\Phi}^{(h)}(\bfd_j).
\end{align*}
\end{enumerate}
\section{NYC Education}
This appendix contains figures that showcase the importance metrics for both structured and unstructured features for each model assessed. We also include the marginal influence of individuals over 25 with higher educational attainment, represented on the original scale.

\begin{figure}
\begin{center}
\centering\includegraphics[width= 19cm]{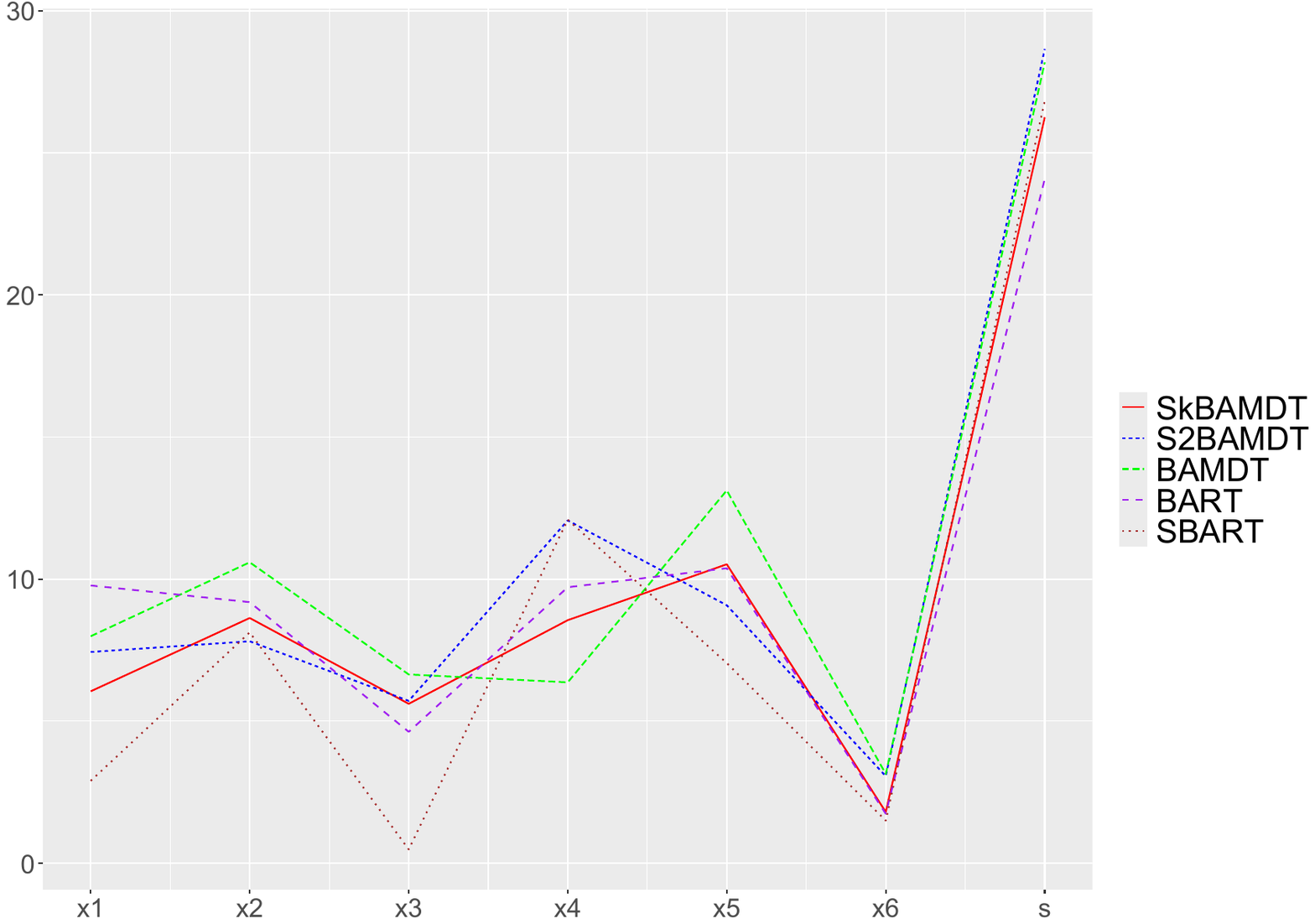}
\end{center}
\caption{The importance metric for each feature given by Sk-BAMDT, S2-BAMDT, BAMDT, BART and SBART. \label{fig:VarC_NY}}
\end{figure}

\begin{figure}
\begin{center}
\centering\includegraphics[width= 14cm]{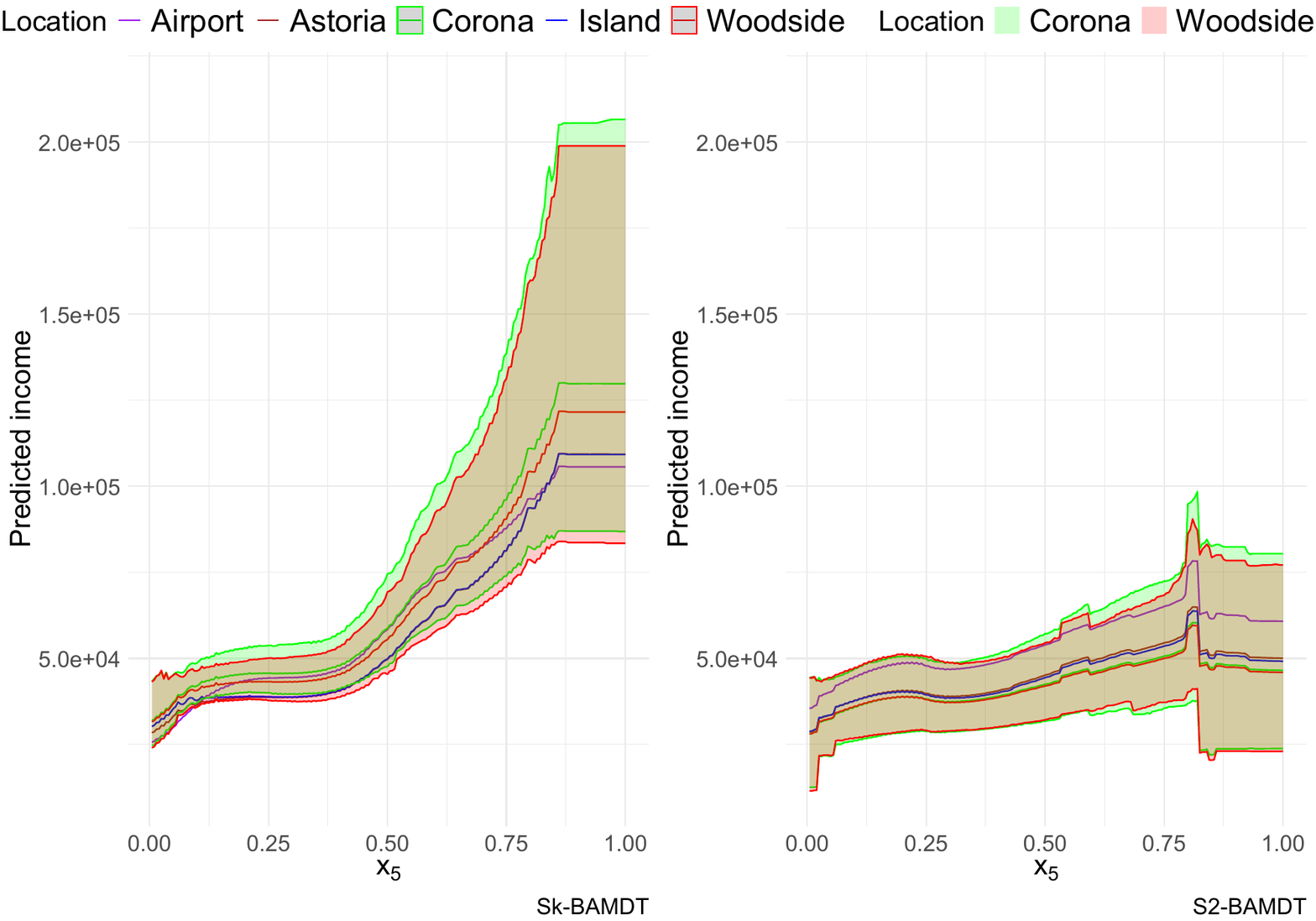}
\end{center}
\caption{ The predicted income versus the population over 25 with at least a bachelor's degree. Colored ribbons represent 95\% predictive credible intervals of two representative locations. \label{fig:VarCX_NY}}
\end{figure}

\section{Simulation study on synthetic data}
This appendix displays a figure that presents the significance metrics for both structured and unstructured features across all models, utilizing one of the fifty simulated datasets for the U-shape example.

\begin{figure}[h]
\begin{center}
\centering\includegraphics[width= 14cm]{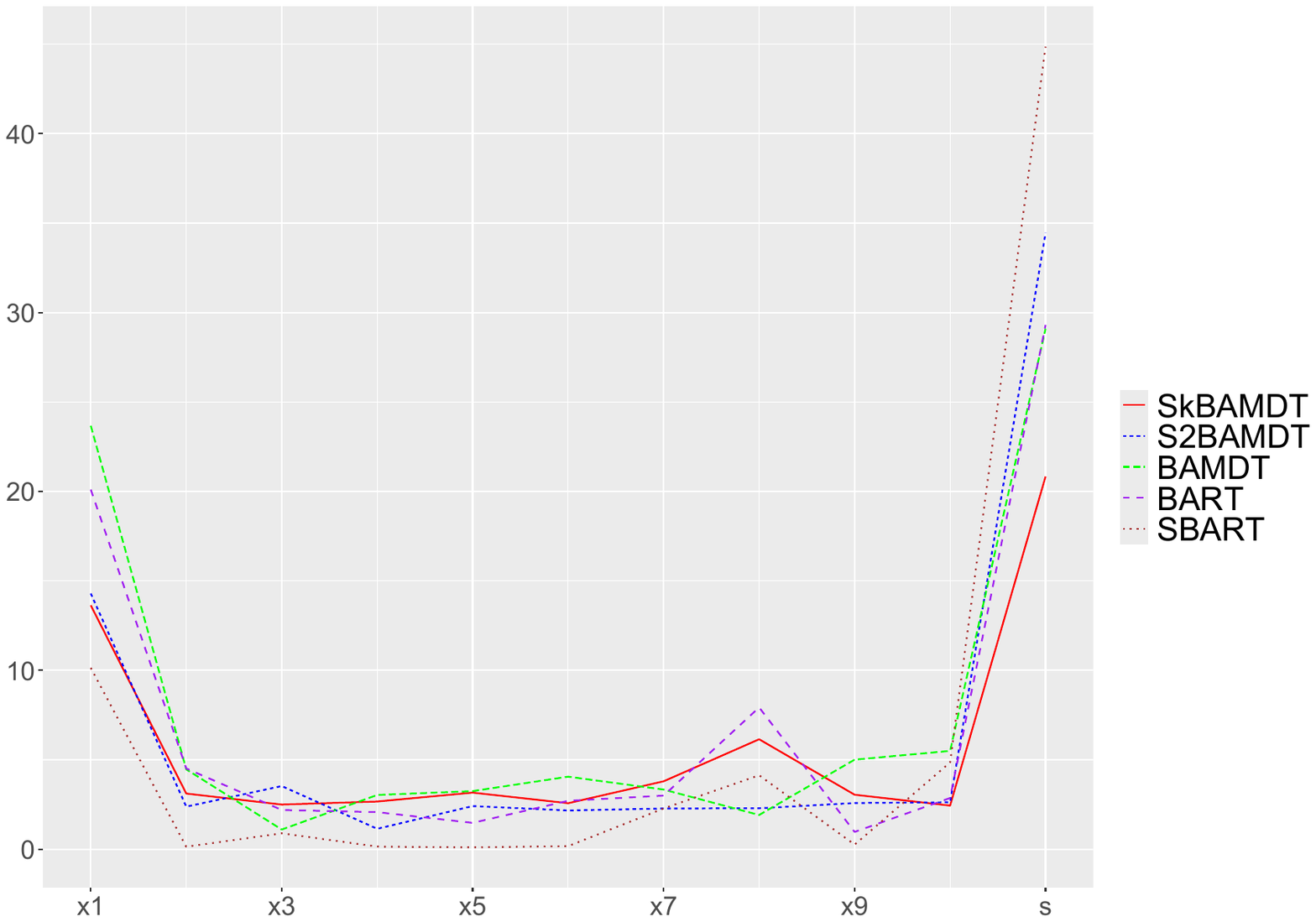}
\end{center}
\caption{The importance metric for each feature for one U-shape simulation. \label{fig:VarC_ScU}}
\end{figure}
\vskip 0.2in
\bibliography{references}

\begin{thebibliography}{41}
\providecommand{\natexlab}[1]{#1}
\providecommand{\url}[1]{\texttt{#1}}
\expandafter\ifx\csname urlstyle\endcsname\relax
  \providecommand{\doi}[1]{doi: #1}\else
  \providecommand{\doi}{doi: \begingroup \urlstyle{rm}\Url}\fi

\bibitem[Chipman et~al.(2010)Chipman, George, and McCulloch]{chipman2010bart}
Hugh~A Chipman, Edward~I George, and Robert~E McCulloch.
\newblock Bart: Bayesian additive regression trees.
\newblock 2010.

\bibitem[Freund and Schapire(1997)]{freund1997decision}
Yoav Freund and Robert~E Schapire.
\newblock A decision-theoretic generalization of on-line learning and an application to boosting.
\newblock \emph{Journal of computer and system sciences}, 55\penalty0 (1):\penalty0 119--139, 1997.

\bibitem[Breiman(1996)]{breiman1996bagging}
Leo Breiman.
\newblock Bagging predictors.
\newblock \emph{Machine learning}, 24:\penalty0 123--140, 1996.

\bibitem[Breiman(2001)]{breiman2001random}
Leo Breiman.
\newblock Random forests.
\newblock \emph{Machine learning}, 45:\penalty0 5--32, 2001.

\bibitem[Zhang and H{\"a}rdle(2010)]{zhang2010bayesian}
Junni~L Zhang and Wolfgang~K H{\"a}rdle.
\newblock The {B}ayesian additive classification tree applied to credit risk modelling.
\newblock \emph{Computational Statistics \& Data Analysis}, 54\penalty0 (5):\penalty0 1197--1205, 2010.

\bibitem[Kindo et~al.(2016)Kindo, Wang, and Pe{\~n}a]{kindo2016multinomial}
Bereket~P Kindo, Hao Wang, and Edsel~A Pe{\~n}a.
\newblock Multinomial probit {B}ayesian additive regression trees.
\newblock \emph{Stat}, 5\penalty0 (1):\penalty0 119--131, 2016.

\bibitem[Bleich et~al.(2014)Bleich, Kapelner, George, Jensen, et~al.]{bleich2014variable}
Justin Bleich, Adam Kapelner, Edward~I George, Shane~T Jensen, et~al.
\newblock Variable selection for {BART}: an application to gene regulation.
\newblock \emph{{T}he {A}nnals of {A}pplied {S}tatistics}, 8\penalty0 (3):\penalty0 1750--1781, 2014.

\bibitem[Linero(2018)]{linero2018bayesian2}
Antonio~R Linero.
\newblock Bayesian regression trees for high-dimensional prediction and variable selection.
\newblock \emph{{J}ournal of the {A}merican {S}tatistical {A}ssociation}, 113\penalty0 (522):\penalty0 626--636, 2018.

\bibitem[Chipman et~al.(2021)Chipman, George, McCulloch, and Shively]{chipman2021mbart}
Hugh~A Chipman, Edward~I George, Robert~E McCulloch, and Thomas~S Shively.
\newblock m{BART}: {M}ultidimensional {M}onotone {BART}.
\newblock \emph{Bayesian Analysis}, 1\penalty0 (1):\penalty0 1--30, 2021.

\bibitem[Hill(2011)]{hill2011bayesian}
Jennifer~L Hill.
\newblock Bayesian nonparametric modeling for causal inference.
\newblock \emph{Journal of Computational and Graphical Statistics}, 20\penalty0 (1):\penalty0 217--240, 2011.

\bibitem[Sparapani et~al.(2016)Sparapani, Logan, McCulloch, and Laud]{sparapani2016nonparametric}
Rodney~A Sparapani, Brent~R Logan, Robert~E McCulloch, and Purushottam~W Laud.
\newblock Nonparametric survival analysis using {B}ayesian additive regression trees ({BART}).
\newblock \emph{Statistics in Medicine}, 35\penalty0 (16):\penalty0 2741--2753, 2016.

\bibitem[Bleich and Kapelner(2014)]{bleich2014bayesian}
Justin Bleich and Adam Kapelner.
\newblock Bayesian {A}dditive {R}egression {T}rees with {P}arametric {M}odels of {H}eteroskedasticity.
\newblock \emph{arXiv preprint arXiv:1402.5397}, 2014.

\bibitem[Pratola et~al.(2016)]{pratola2016efficient}
Matthew~T Pratola et~al.
\newblock Efficient {M}etropolis--{H}astings proposal mechanisms for {B}ayesian regression tree models.
\newblock \emph{{B}ayesian {A}nalysis}, 11\penalty0 (3):\penalty0 885--911, 2016.

\bibitem[Murray(2021)]{murray2017log}
Jared~S Murray.
\newblock Log-linear {B}ayesian additive regression trees for multinomial logistic and count regression models.
\newblock \emph{Journal of the American Statistical Association}, 116\penalty0 (534):\penalty0 756--769, 2021.

\bibitem[Lamprinakou et~al.(2023)Lamprinakou, Barahona, Flaxman, Filippi, Gandy, and McCoy]{lamprinakou2023bart}
Stamatina Lamprinakou, Mauricio Barahona, Seth Flaxman, Sarah Filippi, Axel Gandy, and Emma~J McCoy.
\newblock Bart-based inference for poisson processes.
\newblock \emph{Computational Statistics \& Data Analysis}, 180:\penalty0 107658, 2023.

\bibitem[Ročková and van~der Pas(2020)]{10.1214/19-AOS1879}
Veronika Ročková and Stéphanie van~der Pas.
\newblock {Posterior concentration for {B}ayesian regression trees and forests}.
\newblock \emph{The Annals of Statistics}, 48\penalty0 (4):\penalty0 2108 -- 2131, 2020.

\bibitem[Ro\v{c}kov\'a and Saha(2019)]{pmlrv89rockova19a}
Veronika Ro\v{c}kov\'a and Enakshi Saha.
\newblock On {T}heory for {BART}.
\newblock In Kamalika Chaudhuri and Masashi Sugiyama, editors, \emph{Proceedings of the Twenty-Second International Conference on Artificial Intelligence and Statistics}, volume~89 of \emph{Proceedings of Machine Learning Research}, pages 2839--2848. PMLR, 16--18 Apr 2019.
\newblock URL \url{https://proceedings.mlr.press/v89/rockova19a.html}.

\bibitem[Linero and Yang(2018)]{linero2018bayesian}
Antonio~R Linero and Yun Yang.
\newblock Bayesian regression tree ensembles that adapt to smoothness and sparsity.
\newblock \emph{{J}ournal of the {R}oyal {S}tatistical {S}ociety: Series B ({S}tatistical {M}ethodology)}, 80\penalty0 (5):\penalty0 1087--1110, 2018.

\bibitem[Ge et~al.(2019)Ge, Wang, Teh, Wang, and Elliott]{NEURIPS2019_043c2ec6}
Shufei Ge, Shijia Wang, Yee~Whye Teh, Liangliang Wang, and Lloyd Elliott.
\newblock Random tessellation forests.
\newblock In H.~Wallach, H.~Larochelle, A.~Beygelzimer, F.~d\textquotesingle Alch\'{e}-Buc, E.~Fox, and R.~Garnett, editors, \emph{Advances in Neural Information Processing Systems}, volume~32. Curran Associates, Inc., 2019.

\bibitem[Fan et~al.(2016)Fan, Li, Wang, Wang, and Chen]{fan2016ostomachion}
Xuhui Fan, Bin Li, Yi~Wang, Yang Wang, and Fang Chen.
\newblock The {Ostomachion Process}.
\newblock In \emph{Proceedings of the AAAI Conference on Artificial Intelligence}, volume~30, 2016.

\bibitem[Tomita et~al.(2020)Tomita, Browne, Shen, Chung, Patsolic, Falk, Priebe, Yim, Burns, Maggioni, et~al.]{JMLR:v21:18-664}
Tyler~M Tomita, James Browne, Cencheng Shen, Jaewon Chung, Jesse~L Patsolic, Benjamin Falk, Carey~E Priebe, Jason Yim, Randal Burns, Mauro Maggioni, et~al.
\newblock Sparse projection oblique randomer forests.
\newblock \emph{Journal of Machine Learning Research}, 21\penalty0 (104), 2020.

\bibitem[Rainforth and Wood(2015)]{rainforth2015canonical}
Tom Rainforth and Frank Wood.
\newblock Canonical correlation forests.
\newblock \emph{arXiv preprint arXiv:1507.05444}, 2015.

\bibitem[Rodriguez et~al.(2006)Rodriguez, Kuncheva, and Alonso]{rodriguez2006rotation}
Juan~Jos{\'e} Rodriguez, Ludmila~I Kuncheva, and Carlos~J Alonso.
\newblock Rotation forest: A new classifier ensemble method.
\newblock \emph{IEEE Transactions on Pattern Analysis and Machine Intelligence}, 28\penalty0 (10):\penalty0 1619--1630, 2006.

\bibitem[Blaser and Fryzlewicz(2021)]{blaser2021regularizing}
Rico Blaser and Piotr Fryzlewicz.
\newblock Regularizing axis-aligned ensembles via data rotations that favor simpler learners.
\newblock \emph{Statistics and Computing}, 2021.

\bibitem[Blaser and Fryzlewicz(2016)]{blaser2016random}
Rico Blaser and Piotr Fryzlewicz.
\newblock Random rotation ensembles.
\newblock \emph{The Journal of Machine Learning Research}, 2016.

\bibitem[Stone and Gosling(2024)]{stone2024addivortes}
Adam~J Stone and John~Paul Gosling.
\newblock Addivortes:(bayesian) additive voronoi tessellations.
\newblock \emph{Journal of Computational and Graphical Statistics}, \penalty0 (just-accepted):\penalty0 1--19, 2024.

\bibitem[Maia et~al.(2024)Maia, Murphy, and Parnell]{MAIA2024107858}
Mateus Maia, Keefe Murphy, and Andrew~C. Parnell.
\newblock Gp-bart: A novel bayesian additive regression trees approach using gaussian processes.
\newblock \emph{Computational Statistics \& Data Analysis}, 190:\penalty0 107858, 2024.
\newblock ISSN 0167-9473.
\newblock \doi{https://doi.org/10.1016/j.csda.2023.107858}.
\newblock URL \url{https://www.sciencedirect.com/science/article/pii/S016794732300169X}.

\bibitem[Luo et~al.(2022)Luo, Sang, and Mallick]{luo2022bamdt}
Zhao~Tang Luo, Huiyan Sang, and Bani Mallick.
\newblock Bamdt: Bayesian additive semi-multivariate decision trees for nonparametric regression.
\newblock In \emph{International Conference on Machine Learning}, pages 14509--14526. PMLR, 2022.

\bibitem[Crane et~al.(2020)Crane, Livesu, Puppo, and Qin]{crane2020survey}
Keenan Crane, Marco Livesu, Enrico Puppo, and Yipeng Qin.
\newblock A survey of algorithms for geodesic paths and distances.
\newblock \emph{arXiv preprint arXiv:2007.10430}, 2020.

\bibitem[Shi and Malik(2000)]{shi2000normalized}
Jianbo Shi and Jitendra Malik.
\newblock Normalized cuts and image segmentation.
\newblock \emph{IEEE Transactions on pattern analysis and machine intelligence}, 22\penalty0 (8):\penalty0 888--905, 2000.

\bibitem[Coifman and Lafon(2006)]{coifman2006diffusion}
Ronald~R Coifman and St{\'e}phane Lafon.
\newblock Diffusion maps.
\newblock \emph{Applied and computational harmonic analysis}, 21\penalty0 (1):\penalty0 5--30, 2006.

\bibitem[Belkin and Niyogi(2003)]{belkin2003laplacian}
Mikhail Belkin and Partha Niyogi.
\newblock Laplacian eigenmaps for dimensionality reduction and data representation.
\newblock \emph{Neural computation}, 15\penalty0 (6):\penalty0 1373--1396, 2003.

\bibitem[G{\"o}bel and Jagers(1974)]{gobel1974random}
F~G{\"o}bel and AA~Jagers.
\newblock Random walks on graphs.
\newblock \emph{Stochastic processes and their applications}, 2\penalty0 (4):\penalty0 311--336, 1974.

\bibitem[Dunson et~al.(2022)Dunson, Wu, and Wu]{dunson2022graph}
David~B Dunson, Hau-Tieng Wu, and Nan Wu.
\newblock Graph based gaussian processes on restricted domains.
\newblock \emph{Journal of the Royal Statistical Society Series B: Statistical Methodology}, 84\penalty0 (2):\penalty0 414--439, 2022.

\bibitem[Gelman et~al.(2014)Gelman, Hwang, and Vehtari]{gelman2014understanding}
Andrew Gelman, Jessica Hwang, and Aki Vehtari.
\newblock Understanding predictive information criteria for bayesian models.
\newblock \emph{Statistics and computing}, 24:\penalty0 997--1016, 2014.

\bibitem[He and Hahn(2023)]{he2023stochastic}
Jingyu He and P~Richard Hahn.
\newblock Stochastic tree ensembles for regularized nonlinear regression.
\newblock \emph{Journal of the American Statistical Association}, 118\penalty0 (541):\penalty0 551--570, 2023.

\bibitem[Harris et~al.(1963)]{harris1963theory}
Theodore~Edward Harris et~al.
\newblock \emph{The theory of branching processes}, volume~6.
\newblock Springer Berlin, 1963.

\bibitem[Wang et~al.(2024)Wang, He, and Hahn]{wang2024local}
Meijia Wang, Jingyu He, and P~Richard Hahn.
\newblock Local gaussian process extrapolation for bart models with applications to causal inference.
\newblock \emph{Journal of Computational and Graphical Statistics}, 33\penalty0 (2):\penalty0 724--735, 2024.

\bibitem[Hastie and Tibshirani(2000)]{hastie2000bayesian}
Trevor Hastie and Robert Tibshirani.
\newblock {B}ayesian backfitting (with comments and a rejoinder by the authors).
\newblock \emph{Statistical Science}, 15\penalty0 (3):\penalty0 196--223, 2000.

\bibitem[Gneiting and Raftery(2007)]{gneiting2007strictly}
Tilmann Gneiting and Adrian~E Raftery.
\newblock Strictly proper scoring rules, prediction, and estimation.
\newblock \emph{Journal of the American statistical Association}, 102\penalty0 (477):\penalty0 359--378, 2007.

\bibitem[{G}eo{D}a {D}ata and {L}ab()]{GeoDa}
{G}eo{D}a {D}ata and {L}ab.
\newblock URL \url{:https://geodacenter.github.io/data-and-lab/NYC_Tract_ACS2008_12/}.

\end{thebibliography}
\end{document}